\pdfoutput=1
\documentclass[11pt]{article}

\usepackage[final]{acl}

\usepackage{times}
\usepackage{latexsym}
\usepackage{dsfont}
\usepackage[T1]{fontenc}
\usepackage[utf8]{inputenc}

\usepackage{microtype}

\usepackage{inconsolata}

\usepackage{graphicx}
\usepackage{subcaption}

\usepackage{ISG_style}
\usepackage{amsmath}
\usepackage{amssymb}

\usepackage{algorithm}
\usepackage{algorithmic}

\usepackage[toc,page,header]{appendix}
\usepackage{minitoc}

\usepackage[dvipsnames]{xcolor}

\newcommand{\algoname}{{\rm PGB-CT}}

\usepackage{booktabs}   % For nicer horizontal rules
\usepackage{multirow}   % For multi-row cells if needed
\usepackage{diagbox}  % For diagonal split cells
\usepackage{makecell} % Helps with custom cell height
\usepackage{caption}    % Optional, for caption formatting
\usepackage{adjustbox}  % Optional, to help scale a wide table

\usepackage{enumitem}

\usepackage{balance}

\title{Multi-component Causal Tracing in Large Language Models}

\author{
 \textbf{Zirui Yan\textsuperscript{1}},
 \textbf{Dennis Wei\textsuperscript{2}},
 \textbf{Dmitriy A. Katz\textsuperscript{2}},
 \textbf{Prasanna Sattigeri\textsuperscript{2}},
 \textbf{Ali Tajer\textsuperscript{1}}
\\
\\
 \textsuperscript{1}Rensselaer Polytechnic Institute,
 \textsuperscript{2}IBM Research
}

\begin{document}
\maketitle

%%%%%%%%%%%%%%%%%%%%%%%%%%%%%%%%%%%%%%%%%%%%%%
%% Abstract
%%%%%%%%%%%%%%%%%%%%%%%%%%%%%%%%%%%%%%%%%%%%%%
\begin{abstract}
Causal tracing systematically intervenes on a large language model's (LLM's) internal representations to uncover and quantify the causal pathways linking specific inputs or computations to specific metrics of interest, quantifying the LLM's behavior. Building on previous single-component or single-layer studies, this paper presents a unified framework for causally tracing multiple components simultaneously. This framework systematically identifies the subsets of components (e.g., attention heads and multi-layer perceptron neurons) most critical to a desired target performance metric (e.g., accuracy and fairness). This is achieved by incorporating flexible interventions applied to a wide range of desired metrics. To address the combinatorial complexity of the multi-component problem, an efficient algorithm is designed that leverages soft interventions and a carefully designed metric transformation, converting the combinatorial search problem into a continuous one that can be solved efficiently under proper constraints, thereby generating proper binary decisions for selecting components. Experimental results demonstrate that the proposed method efficiently identifies subsets of the model's components that have a high impact on the target metric, outperforming existing baseline approaches. Our code is available at \url{https://github.com/ZiruiYan/multi-component-causal-tracing}.
\end{abstract}
%%%%%%%%%%%%%%%%%%%%%%%%%%%%%%%%%%%%%%%%%%%%%%
% Introduction
%%%%%%%%%%%%%%%%%%%%%%%%%%%%%%%%%%%%%%%%%%%%%%
\section{Introduction}
Large language models (LLMs) have revolutionized natural language processing through their ability to generate fluent, contextually-aware text, with a wide range of applications such as machine translation~\citep{zhang2023prompting}, text summarization~\citep{van2024adapted}, code generation~\citep{jiang2024self}, and agentic tasks~\citep{xi2025rise}. However, LLMs are prone to various forms of safety risks. They can inadvertently learn and propagate societal biases~\citep{sheng2021societal}, produce factually incorrect statements~\citep{augenstein2024factuality}, and even generate harmful or deceptive content through jailbreak attacks~\citep{wang2025align}. These challenges, coupled with the fundamental interest in revealing how LLM knowledge flows, highlight the need for interpretability capabilities to systematically analyze and elucidate the underlying mechanisms of these models.

\begin{figure*}[htbp]
    \centering
    \includegraphics[width=0.4\textwidth]{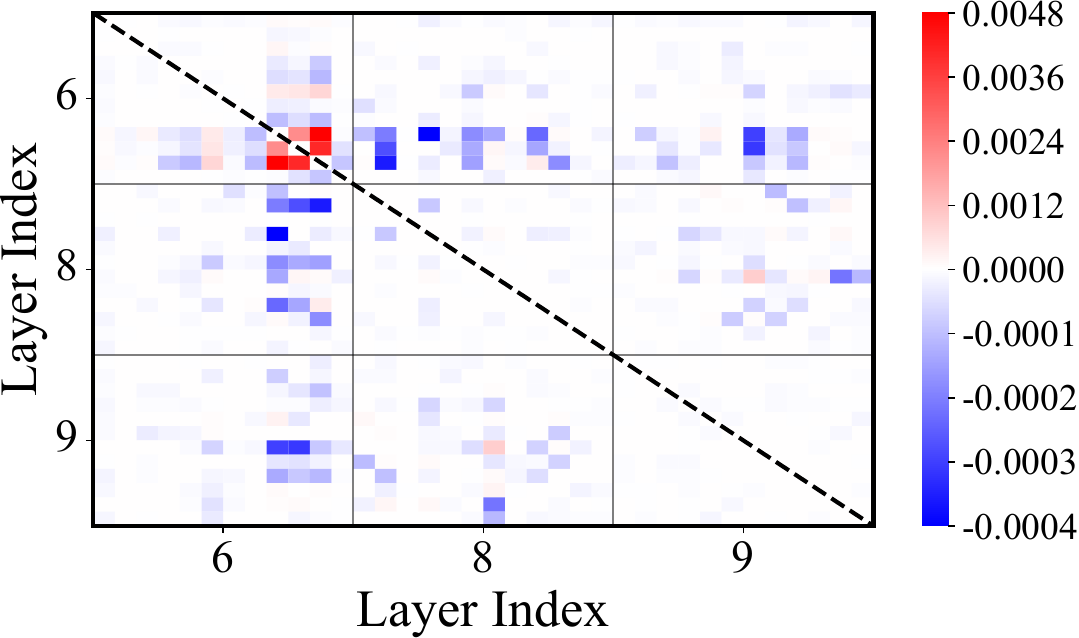}\hspace{0.45 in}
    \includegraphics[width=0.4\textwidth]{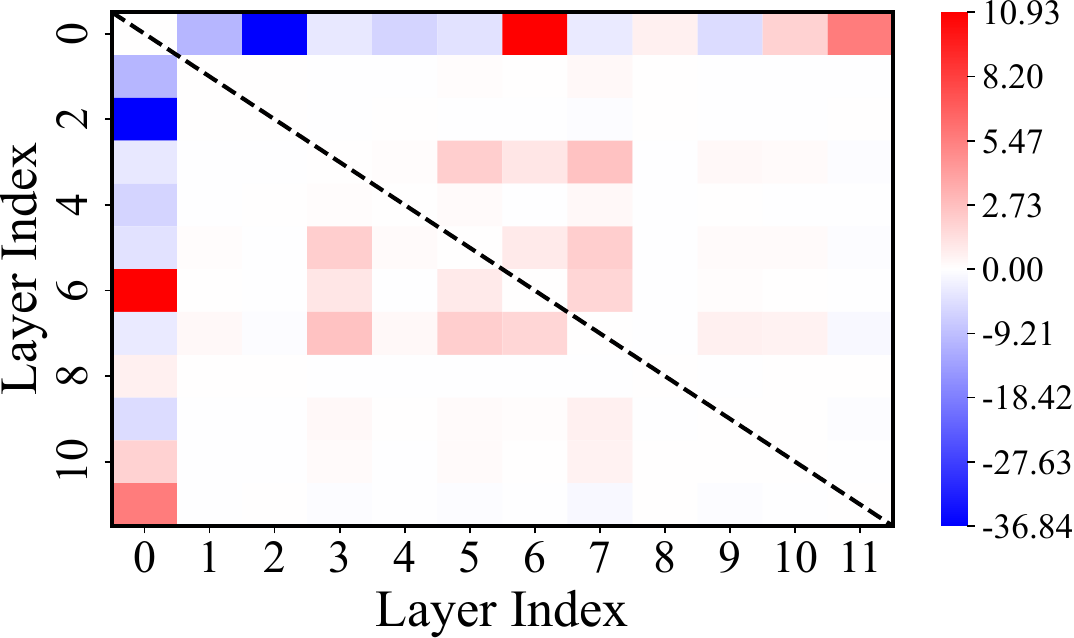} 
    \vspace{-0.05 in}
    \caption{The deviation from linearity due to intervening on two attention heads within layers $\{6,8,9\}$ of GPT2-small on the WinoBias dataset (left), and two MLP layers of GPT2-small on the Professions dataset (right).}
    \label{fig:results_toy}
    \vspace{-0.15 in}
\end{figure*} 

Several interpretability approaches have been presented recently~\citep{luo2024understanding}, ranging from external explanations that focus on the impact of the prompt on a single prediction (e.g., gradient-based~\citep{sikdar-etal-2021-integrated} and vector-based~\citep{chen-etal-2020-generating-hierarchical,modarressi-etal-2022-globenc, paes2025mexgen}), to internal explanations that aim to uncover general knowledge and representational structures within a model (e.g., probing~\citep{petroni2019language,li2024inference}, circuit analysis~\citep{chughtai2023toy,wang2023interpretability, conmy2023towards}, and sparse autoencoders~\citep{bricken2023monosemanticity, cunningham2023sparse}). 

Among the global methods, \emph{causal tracing}~\citep{vig2020investigating}, which is also called activation patching~\citep{davies2023discovering} or interchange interventions~\citep{geiger-etal-2020-neural,geiger2022inducing}, has emerged as an effective strategy by treating LLMs as causal models, enabling systematic interventions on internal values to pinpoint which specific model components most strongly influence a chosen metric. This goes beyond counterfactual explanations for black-box models~\citep{wachter2017counterfactual,karimi2022survey}, and recent work has proposed different methods~\citep{mohebbi-etal-2023-quantifying} and metrics~\citep{zhang2024towards}. By revealing the importance of specific network components, causal tracing enables more targeted editing, fine-tuning, or enhancement of LLMs~\citep{geiger2021causal, meng2022locating,meng2022mass,cai2024locating,hase2024does, chen2026can, prakash2024fine}. For example, \citet{cai2024locating} uses causal mediation analysis to identify model components that contribute most to gender bias and then mitigate gender bias by intervening on those components; \citet{chen2026can} show that tracing-guided editing can rival fine-tuning and in-context editing; \citet{prakash2024fine} show that fine-tuning improves performance by strengthening existing circuits rather than creating new ones through causal tracing.

Most existing studies in causal tracing have focused on analyzing single components in isolation: an individual neuron~\citep{vig2020investigating,vig2020causal} or an entire layer~\citep{meng2022locating,meng2022mass,hase2024does}. This overlooks the potentially \emph{non-linear} effects arising from interactions among multiple components. Emergent multi-component behaviors have been revealed in circuit analysis. \citet{elhage2021mathematical} and~\citet{olsson2022context} discovered the \emph{induction heads} mechanism, according to which two attention heads in different layers jointly increase the likelihood of a specific token. We refer to~\citet{ferrando2024primer} for additional findings.

A phenomenon similar to induction heads can also be observed in the causal tracing of LLMs, in which jointly intervened-on pairs of attention heads or multi-layer perceptron (MLP) neurons exhibit a pronounced \emph{non-linear} impact on the model's outputs. Figure~\ref{fig:results_toy} illustrates this phenomenon in GPT2-small on two datasets, where selecting specific subsets of components can significantly boost the metric of interest (which is discussed in detail in Appendix~\ref{sec:toy}). The figure shows the metric values obtained by intervening on two components simultaneously, minus the sum of the metric values obtained when intervening on each component individually. Red indicates a positive difference (the joint intervention enhances the metric), whereas blue indicates a negative difference (the joint intervention diminishes the metric). These nonlinear effects violate the observation of \citet{syed2023attribution, hanna2024have,shah2024decomposing} for these specific models and metrics, implying that linear approximations do not hold in general.

Motivated by these observations, we argue that causal tracing in LLMs should be pursued through a \emph{multi-component} perspective. Although existing studies provide valuable evidence, they lack a unified framework for defining the problem and rigorously accounting for non-linear interactions. We address this gap by formulating the problem of \emph{multi-component causal tracing}, in which the goal is to identify a subset of components (e.g., attention heads, MLP neurons) that collectively maximize a chosen metric of interest. Recent work by \citet{davies2023discovering} investigated a similar problem via gradient-based search.
However, our analysis in Appendix~\ref{sec:DCM} indicates that their penalty and optimization choices drive much of the performance gap. The relative efficiency and effectiveness of gradient-based algorithms compared with alternative methods remain unknown. We fill these gaps, and our main contributions are 
as follows.
\begin{itemize}[leftmargin=0.15in, itemsep=0.01in, topsep=0.02in]
    \item \textbf{Unified Framework.} We formalize a causal tracing approach that systematically traces causal effects across multiple components of LLMs. This framework provides general definitions of intervention and metrics that can be tailored for various tasks, including probing linguistic features, testing bias, and monitoring factual correctness. The multi-component causal tracing problem poses a combinatorial challenge, resulting in a search space that grows exponentially with model size.\vspace{-.1 in}
    \item \textbf{Efficient Algorithm.} To efficiently search in the combinatorial space, we perform \emph{soft} interventions on components, combined with a 
    carefully designed
    metric transformation and scheduled penalty function. The transformation unlocks the power of gradient descent, while the penalty function ensures sparsity and binary values. Furthermore, the scheduling in the penalty function allows the algorithm to converge to the desired sparsity level at an appropriate rate and to stop once that level is achieved.
    \item \textbf{Empirical Efficiency and Effectiveness.} We show that top-$k$ and greedy baselines can be computationally expensive, as they require computing multiple forward runs, at least one for each component on all samples. In our experiments on GPT2-medium with the WinoGender dataset, our method achieves a $ 1.76$-times speedup compared to top-$k$ and a $ 229$-times speedup compared to greedy, while maintaining performance close to greedy and superior to top-$k$.
\end{itemize}

\begin{figure*}
    \centering
    \includegraphics[width=0.85\linewidth]{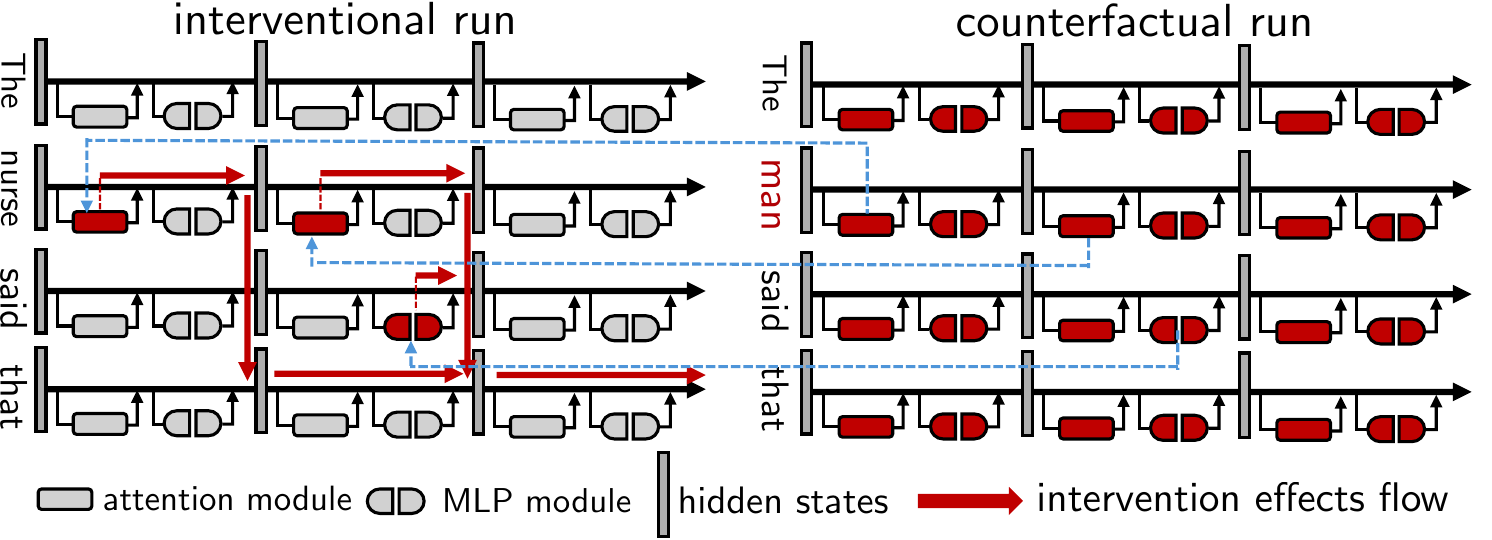}
    \vspace{-0.05 in}
    \caption{Counterfactual intervention in LLMs.}
    \label{fig:toy_structure}
    \vspace{-0.15 in}
\end{figure*}

%%%%%%%%%%%%%%%%%%%%%%%%%%%%%%%%%%%%%%%%%%%%%%
% Problem definition
%%%%%%%%%%%%%%%%%%%%%%%%%%%%%%%%%%%%%%%%%%%%%%%%%%%%%%%%%%%%%%%%%%%%%%%%%%%%%%%%%%%%%%%%%%%%
% LLM
\section{LLM Model \& Notations}
\label{sec:llm}
\noindent\textbf{Pre-trained LLM.} Consider a pre-trained LLM $f_\theta$, where $\theta$ compactly parameterizes the LLM. The model receives the prompt $\bs$, operates over the vocabulary set $\mcV$, and generates words sequentially. The word generated at time $t$ is denoted by $v_t$, and accordingly, we define the sequence generated up to $t$ as the ordered sequence $\bv^{t}\triangleq \{v_i\in\mcV: i\in[t]\}$. Word generation begins with an initial \emph{ordered} sentence of $n$ words denoted by $\bv^n$ and proceeds by generating subsequent tokens sequentially and autoregressively for indices $t>n$. Specifically, at time $t$, the token is generated according to 
\begin{equation} 
\label{eq:greedy} 
\setlength\abovedisplayskip{8pt}
   \setlength\belowdisplayskip{8pt}
v_t \; \triangleq \; \arg\max_{v \in \mathcal{V}} \P(v \mid \bv^{t-1}) \ , 
\end{equation}
where $\P(v \mid \bv^{t-1})$ is an \emph{unknown} conditional probability mass function with the following shorthand:
\begin{equation}
\label{eq:conditional} 
\setlength\abovedisplayskip{8pt}
   \setlength\belowdisplayskip{8pt}
    \P_t(v) \; \triangleq \; \P(v \mid \bv^{t-1})\ , \quad  \forall \; v\in\mcV \ .
\end{equation}
This token-generation process continues sequentially until a predefined stopping criterion is met, such as the generation of a predetermined end-of-sequence token.

In designing our algorithm presented in Section~\ref{sec:algorithm}, we leverage the average likelihood of the tokens generated over specified intervals. For this purpose, we define the following mixture distribution, which for any given $t,d\in\N$, averages the conditional probability mass functions (PMFs) over $\{t,\dots, t+d\}$, i.e., 
\begin{equation}
\setlength\abovedisplayskip{8pt}
   \setlength\belowdisplayskip{8pt}
    \P^{\sf avg}_{t:d}(\bv^{t+d}) \; \triangleq \;\frac{1}{d+1} \sum_{i=0}^{d} \P_{t+i}(v_{t+i}) \ .
\end{equation}
The conditional PMFs can be aggregated in different ways, such as by the geometric mean.

%%%%%%%%%%%%%%%%%%%%%%%%%%%%%%%%%%%%%%%%%%%%%%
% Components of LLM
\vspace{0.04 in}
\noindent\textbf{LLM Components.}  LLMs are constructed by stacking multiple computational blocks. Each block consists of several key components, including attention heads, feed-forward layers,  MLP layers, and normalization layers. We denote the set of components available for our analysis by
\begin{equation}
   \setlength\abovedisplayskip{8pt}
   \setlength\belowdisplayskip{8pt}
    \mcC \; \triangleq \; \{c_i: i \in [N]\} \ .
\end{equation}
This set can include any non-overlapping substructures of the model, allowing flexibility in analyzing specific components.
Without loss of generality, we assume the components in $\mcC$ are ordered following the order they appear in the LLM's architecture. 

Some studies suggest that key knowledge in LLMs is primarily stored in attention heads and MLP layers (see, e.g.,~\citep{clark2019does,vig2020investigating, geva2020transformer, geva2022transformer,meng2022locating}). Hence, this paper focuses primarily on the attention heads and the MLP layers. Specifically, we are interested in each head's attention weights and the behavior of individual neurons after the down-project layer. In our empirical analysis, we consider models ranging from DistilGPT2 ($72$ attention heads and $4608$ MLP neurons) to GPT2-xl ($1200$ attention heads and $76800$ MLP neurons), as well as new architectures of similar size.

\vspace{0.04 in}
\noindent\textbf{Forward Mapping.}
For each component $c_i\in \mcC$, we denote the input hidden states by $\bg_{i}$, the output hidden states by $\bh_{i}$, and the component mapping function by $f_i$, which maps out $\bg_i$ to $\bh_i$, i.e., 
\begin{equation}
\label{eq:fwd}
\setlength\abovedisplayskip{8pt}
   \setlength\belowdisplayskip{8pt}
    \bh_{i} \; = \; f_i\big(\bg_{i}\big) \ , \qquad \forall i\in[N]\ .
\end{equation}
We use the notations $\bg\triangleq (\bg_{i},\cdots, \bg_{N})$ and $\bh\triangleq (\bh_{i},\cdots, \bh_{N})$ to compactly represent all the input and output states, respectively. 

\section{Multi-Component Causal Tracing}
\label{sec:CT}
%%%%%%%%%%%%%%%%%%%%%%%%%%%%%%%%%%%%%%%%%%%%%%
% Causal Mediation Analysis
\noindent\textbf{Causal Mediation Analysis.} Our central objective is to quantify how changes in specific components of an LLM influence an LLM's behavior in specific contexts (e.g., biases).  To formalize a principled approach to trace the impact of a change in a component, we perform \emph{causal mediation analysis}~\citep{pearl2022direct}. This enables examination of how an \emph{interventional treatment} (i.e., a change in a component) affects the outcome (i.e., a target metric) through variations in intermediate variables, known as \emph{mediators} (i.e., intermediate components). In the setting of LLMs, the mediators are various components of the LLM, such as attention heads and MLP neurons, where the intervention targets can be attention weights and MLP down-projection.  Our focus is on these indirect effects: specifically, how our treatment, intervening on mediators, changes the metric of interest.

We denote the subset of components selected for treatment by $\mcH\subseteq \mcC$. To specify the components selected for treatment, we define $\{m_i:i\in[N]\}$ such that $m_i\triangleq\mathds{1}\{c_i\in\mcH\}$, where $\mathds{1}$ is the indicator function. Accordingly, we define $\bm\triangleq (m_i,\dots, m_N)$. In this context, a treatment involves intervening on these components by replacing specific attention weights or neuron activations with counterfactual ones obtained from a separate experiment. This procedure is described next.

\vspace{0.04 in}
\noindent\textbf{Intervention via Mixture Forward.} We formalize and perform a process referred to as a \emph{mixture forward pass} to analyze how an intervention on the set $\mcH$ propagates throughout the rest of the model components and eventually influences a target metric. This process is demonstrated in Figure~\ref{fig:toy_structure} and it is similar to activation patching when applied only to activations~\cite{davies2023discovering}. For this purpose, we specify the following two word-generation processes.

\begin{enumerate}[leftmargin=0.15in, itemsep=0.03in, topsep=0.04in]
    \item \emph{Counterfactual Run:} Corresponding to a given $\bs$, we create a \emph{counterfactual} prompt, denoted by $\bs'$, by modifying some of the elements of the prompt $\bs$ (e.g., altering a word in the prompt or adding noise to word embeddings). This modified prompt generates a new set of hidden states $\bg'$ and $\bh'$, which are related according to \eqref{eq:fwd}.
    \item \emph{Interventional Run:} Given a prompt $\bs$ and its associated counterfactual $\bs'$ and a set of components $\mcH$, we create an intervention by replacing the hidden states $\{\bh_i: m_i=1, i\in [N] \}$ with those associated with $\bs'$, i.e., $\{\bh'_i: m_i=1, i\in [N] \}$. These changes, in turn, induce changes in all the downstream component states, distinct from those associated with $\bs$ and $\bs'$. To emphasize this, we define $\bar\bg_i$ and $\bar\bh_i$ as the post-intervention hidden states of the components. We note that the state transformation functions $\{f_i:i\in[N]\}$ remain unchanged. Hence,
    \begin{equation}
\label{eq:mixfwd}
    \bar\bh_{i} \; \triangleq \;  (1-m_i) \cdot f_i(\bar\bg_{i})   + m_i \cdot \bh'_{i} \ ,
\end{equation}
\end{enumerate}
Finally, similarly to $v_t$ and $\bv^t$, we define $u_t$ and $\bu^t$ as the sequence generated by the interventional run and denote the associated conditional PMF in the interventional run and the average conditional PMF over the interval $\{t,\dots, t+d\}$ by
\begin{align}
    \bar{\P}_t(u_t) \; \triangleq  &\; \P_{\theta}(u_t \mid \bu^{t-1}, \bm) \ , \\
    \bar{\P}_{t:d}^{\sf avg}(\bu^{t+d}) &\; \triangleq \;\frac{1}{d+1} \sum_{i=0}^{d} \bar{\P}_{t+i}(u_{t+i}) \ .
\end{align}
Throughout the rest of the paper, we refer to $\P_t$ and $\bar{\P}_t$ as the {\em observational} and {\em interventional} distributions, respectively. 

\vspace{0.04 in}
\noindent\textbf{Metric.} The target metric $\ell(\bs, \bx, \by)$ compares the observational and interventional distributions evaluated on a prompt $\bs$, a continuation $\bx$, and possibly an alternative continuation $\by$.

\begin{example}
% [Gender bias.] 
\label{exp:genderbias}
\emph {
\textbf{(Gender bias.)}
Consider the example of gender bias similar to the one in Figure~\ref{fig:toy_structure}, in which we have the prompt $\bs=$\emph{``The technician told the customer that she''}. The counterfactual intervention replaces ``\emph{she}'' with ``\emph{he}'', producing the counterfactual prompt $\bs'=$``\emph{The technician told the customer that he}''. A model without gender bias is desired that would associate the two roles equally with the pronoun ``\emph{she}''.
Specifically, given a stereotypical continuation $\bx=$``\emph{could pay with cash.}'' and an anti-stereotypical one $\by=``$\emph{had completed the repair.}'', it is desirable to have equal completion probabilities
\begin{equation}
    \P^{\sf avg}_{t:d}(\bx) \; = \; \P^{\sf avg}_{t:d}(\by) \ .
\end{equation}
In this context, to quantify the bias, one can use the conditional likelihood ratio of generating $\bx$ versus~$\by$ \cite{vig2020investigating}, i.e.,
\begin{equation}
    \P_{\theta}(\by\mid \bs) \; \big/ \; \P_{\theta}(\bx\mid \bs)  \ .
\end{equation}
When this ratio is smaller (larger) than 1, the model behaves stereotypically (anti-stereotypically), and when the ratio is 1, it is considered unbiased. Leveraging this notion of unbiasedness, we can set the metric $\ell$ as follows.}

{
\fontsize{10pt}{11pt}\selectfont
\begin{equation}
\label{eq:metricgenderbias}
\ell(\bs, \bx, \by)  =  \frac{\bar{\P}_{|\bs|+1:|\by|}^{\sf avg}(\by)/\bar{\P}_{|\bs|+1:|\bx|}^{\sf avg}(\bx)}{\P_{|\bs|+1:|\by|}^{\sf avg}(\by)/\P_{|\bs|+1:|\bx|}^{\sf avg}(\bx)} -1 \ .
\end{equation}
}
\emph {The value $\ell(\bs, \bx, \by)=0$ means that the components intervened on in $\bar{\P}^{\sf avg}$ do not affect gender bias; $\ell(\bs, \bx, \by)>0$ indicates that they store stereotype information; and $\ell(\bs, \bx, \by)<0$ indicates that they store anti-stereotype information.
}
\end{example}

\begin{example}
% [Knowledge Localization.]
\emph{ 
\textbf{(Knowledge Localization.)} 
Consider the setting in which we seek to identify the critical components that enable the LLM to answer knowledge-based questions. For example, consider the prompt $\bs=$ \emph{``Steve Jobs was the founder of''} with the desired output being that the model assigns full probability to $\bx=$\emph{``Apple''}:}
\begin{equation}
    \mathbb{P}_{\theta}(\bx \mid \bs) \; = \; 1 \ .
\end{equation}
\emph {For this example, the counterfactual run is generated by adding a large amount of noise to the embedding layer, disrupting knowledge retrieval. Hence, an objective can be to measure how well the model answers the question based on the stored knowledge in its components. For this purpose, we can use the following metric:
\begin{equation}
 \ell(\bs, \bx, \by) = \P_{|\bs|+1:0}^{\sf avg}(\bx) - \bar{\P}_{|\bs|+1:0}^{\sf avg}(\bx) \  .
\end{equation}
In this case, we only use one continuation $\bx$. Higher values of $\ell(\bs, \bx, \by)$ indicate that the components intervened in $\mcH$ are crucial for knowledge retention, while a low $\ell(\bs, \bx, \by)$  suggests that these components do not store factual information.
}
\end{example}

%%%%%%%%%%%%%%%%%%%%%%%%%%%%%%%%%%%%%%%%%%%%%%
% Causal Tracing Problem
\noindent\textbf{Causal Tracing Objective.} The objective is to identify a set of at most $S\in\Z^+$ components that have the highest joint contribution to a metric of interest. To formalize, denote the available dataset by $\mcD\triangleq \{\bs_i, \bx_i, \by_i\}_{i=1}^{M}$, where each $\bs_i$ is an original input sentence, and $\bx_i$ and $\by_i$ are two possible continuations for it ($\by_i$ may be unused as in Example 2 above). Given $\mcD$, the set of intervened components specified by $\bm$, and the associated observational and interventional distributions $\P_t$ and $\bar{\P_t}$, we define the average metric 
\begin{equation} 
\label{eq:avg_metric} 
\ell(\mcD,\bm ) \; \triangleq  \; \frac{1}{M}   \sum_{i=1}^M\;  \ell(\bs_i, \bx_i ,\by_i)\ .
\end{equation} 
Hence, the objective of causal tracing is 
\begin{equation} 
\label{eq:problem}
\max_{\bm \in \{0,1\}^N} \; \ell(\mcD,\bm)  \quad \text{s.t.} \quad \|\bm\|_0 \leq S \ , 
\end{equation}
where the sparsity parameter $S\in\Z^+$ is a pre-specified constraint chosen by the user, which can be set according to the user's practical needs and constraints. This constraint ensures greater interpretability by isolating a small, meaningful subset of model components for targeted analysis and for later knowledge editing or fine-tuning. We note that when $S=1$, the problem reduces to a single-component causal tracing problem that has already been investigated.  We define $\frac{S}{N}$ as the \emph{sparsity level} of an intervention. The problem can also be cast as selecting a minimum-cardinality subset achieving a target performance, and our algorithm can handle both formulations.
\vspace{-0.02 in}
\begin{theorem}[Search Complexity]
\label{thm:complexity}
Given positive integers $N$ and $S$, the size of the search space in \eqref{eq:problem}, denoted by $C_N(S)$, scales as follows.
\begin{enumerate}[leftmargin=0.15in, itemsep=0.00in, topsep=0.00in]
    \item \textbf{$S$ is a fixed constant:} $C_N(S)$ grows polynomially in $N$, i.e.,  
    $ C_N(S) \; = \; \Theta\bigl(N^S\bigr)\;$. \vspace{-0.02 in}
    \item \textbf{$S$ is proportional to $N$:} when $S = \alpha N$ for a  constant $\alpha \in (0,1)$, $C_N(S)$ grows exponentially in $N$ (see Appendix~\ref{pf:complexity} for details).
\end{enumerate}
\end{theorem}

%%%%%%%%%%%%%%%%%%%%%%%%%%%%%%%%%%%%%%%%%%%%%%
% Algorithm
\section{Penalized Gradient-based Causal Tracing (\algoname) Algorithm}
\label{sec:algorithm}
As 
% specified by 
Theorem~\ref{thm:complexity} shows that the discrete search space for the problem defined in \eqref{eq:problem} is prohibitive even for the smallest models, in which $N$ can be in the order of $10^{4}$. To circumvent this search complexity, we design the \textbf{P}enalized \textbf{G}radient-\textbf{B}ased \textbf{C}ausal \textbf{T}racing (\algoname) algorithm. The algorithm converts the discrete problem into a continuous one, incorporates an appropriate reward transformation to prevent exploding gradients, and adds a regularization function. The algorithm stops when the desired sparsity parameter $S$ is achieved. Next, we describe the key ideas of the algorithm. 
The algorithm is summarized in Algorithm~\ref{alg:CT} in Appendix~\ref{sec:code}. 

\vspace{.04 in}
\noindent\textbf{Continuous Relaxation.} To circumvent the combinatorial complexity of the search space, we relax the original subset selection into a soft subset selection framework. That is, the masking vector $\bm$ falls in the continuous space $[0,1]^N$. Under this relaxation, $m_{i} = 1$ specifies \emph{full} intervention on component $c_i$, where the hidden states are replaced by the counterfactual ones. In contrast, $m_{i} = 0$ specifies no intervention on $c_i$, and the post-intervention hidden states of $c_i$ is computed as $\bar\bh_{i}=f_i(\bar\bg_{i})$. Besides these extreme cases, $m_{i}\in(0,1)$ indicates a linear mixing of two states using \eqref{eq:mixfwd}, and we use the same set of notations for conditional probabilities and metrics.

\vspace{.04 in}
%%%%%%%%%%%%%%%%%%%%%%%%%%%%%%%%%%%%%
% Relaxed Problem
\noindent\textbf{Transformed Reward.} Based on the discussion above, the penalized counterpart of the original problem in~\eqref{eq:avg_metric} with a continuous relaxation becomes
\begin{equation}
\label{eq:loss_ori}
    \max_{\bm \in [0,1]^N} \ell(\mcD, \bm)  - {\sf reg}(\bm) \ ,
\end{equation}
where ${\sf reg}(\bm)$, specified in the next paragraph, is a penalty term that encourages binary and sparse solutions. We note that, in practice, the metric $\ell(\mcD,\bm)$ can be unbounded (e.g., in \eqref{eq:metricgenderbias}), which can lead to unbounded gradients and is not amenable to gradient-based optimization. To address this, we apply an inverse transformation to the metric to maintain stable gradient behavior and solve the following problem.
\begin{equation}
\label{eq:loss}
\resizebox{0.87\linewidth}{!}{$\displaystyle
    \min_{\bm \in [0,1]^N}  \mcL (\mcD,\bm)  \triangleq  \frac{1}{1+\ell(\mcD, \bm)}  + {\sf reg}(\bm) \; .
    $}
\end{equation}

%%%%%%%%%%%%%%%%%%%%%%%%%%%%%%%%%%%%%%%% 
% Regularization Function
\noindent\textbf{Scheduled Penalty Function.} A standard regularization approach involves penalizing the $\ell_0$ norm to control the sparsity. However, the $\ell_0$ norm is not differentiable and often yields values between $0$ and $1$, causing performance degradations when binarized. To address this, we introduce the following penalty term, which encourages binary masking:
\begin{equation}
\label{eq:reg}
{\sf reg}(\bm) \; = \; \lambda_1 \|\bm\|_1 + \lambda_2 \bm^{\top} (\mathbf{1}-\bm) \ ,
\end{equation}
where $\lambda_1,\lambda_2\in\R_+$ are regularization parameters and $\mathbf{1}$ is an all-one vector. The first term serves as an approximation for the $\ell_0$ norm, while the second term penalizes non-binary values. This term is maximized when $\bm = 0.5\cdot\mathbf{1}$ and becomes zero when $\bm$ is binary. The choice of the penalty term is motivated by the behavior of its gradient.  As we restrict the values of each $\{m_i\}$ to be within $[0,1]$, the gradient of 
% ${\sf reg}(\bm)$
the regularization term
takes the form
\begin{equation}
\label{eq:Reggradient}
        \nabla {\sf reg}(\bm) \; = \; \lambda_2 \Bigl(\bigl(1+\frac{\lambda_1}{\lambda_2}\bigr)\mathbf{1}-2\bm\Bigr)\ .
\end{equation}
This gradient involves a directional force imposed by the term $-2\bm$: if $2m_i>1+\frac{\lambda_1}{\lambda_2}$, it encourages $m_i$ to move toward $1$; and if $2m_i<1+\frac{\lambda_1}{\lambda_2}$, it encourages $m_i$ to move toward $0$. Hence, the encouraged direction is specified by  $\lambda_1/\lambda_2$. To prevent early convergence to suboptimal local minima, we adopt a dynamic scheduling strategy for the regularization parameters $\lambda_1$ and $\lambda_2$, gradually increasing their values over the course of training. Desired sparsity and binarity are guaranteed when both regularization parameters are large enough. This encourages the optimizer to explore more broadly in the early stages and remove the components gradually. We note that the penalty function we design does not include $S$.

\vspace{0.04 in}
\noindent\textbf{Gradient Descent and Discretization.} Finally, we adopt gradient descent to minimize $\mcL(\mcD,\bm)$ in \eqref{eq:loss}. We denote the value of $\bm$ in iteration $t$ of gradient descent with truncation by $\bm_t$, and it is computed as follows:
\begin{align}
\label{eq:gradiendescent}
    \bm_{t+1} & = \bm_{t} - \eta_t \nabla \mcL_t(\mcD,\bm_t)\ , \\
    \bm_{t+1} & = \min\{\max\{\bm_{t+1},\mathbf{0}\},\mathbf{1}\} \ ,
\end{align}
where $\eta_t$ is the learning rate in iteration $t$ and $\mcL_t$ is the loss function calculated based on $\bm_t$ and mini-batch.

We apply truncation to obtain a binary mask by a threshold $\tau$ and check the sparsity after each epoch:
\begin{equation}
\label{eq:discrete}
    \mcH_{t+1} = \{c_i\in[N] \;:\; m_{t+1,i}>\tau\} \ ,
\end{equation}
where $\bm_{t+1}=(m_{t+1,1}\dots,m_{t+1,N})$. The algorithm then measures the sparsity of $\mcH$ and stops when $|\mcH_{t+1}|\leq S$. In our experiments, we observe that setting $\tau=0.5$ gives good performance across various datasets.

We note that the Desiderata-based Component Masking (DCM) algorithm~\cite{davies2023discovering} also relies on soft masking and gradient-based optimization, but differs from our approach in both its reward formulation and regularization design.

\vspace{0.04 in}
\noindent\textbf{Reward:} DCM uses the raw metric $\ell(\mcD,\bm)$ directly as the reward, without any transformation such as $\frac{1}{1+\ell(\mcD,\bm)}$, and is applied only to activation-layer components with respect to next-token target probabilities. Using raw metrics can make optimization harder to calibrate, since different metrics may have very different scales, and even the same metric may vary in scale across training stages (e.g., when considering MLP neuron ratios in Figures~\ref{fig:results_MLP} and~\ref{fig:results_factual}). As a result, it becomes more difficult to choose and interpret a consistent regularization strength across settings. In contrast, our transformed reward $\frac{1}{1+\ell(\mcD,\bm)}$ keeps the objective within a more stable numerical range, which empirically makes regularization tuning more robust across metrics and experiments.

\noindent\textbf{Penalty function:} DCM adopts the $\ell_{0.5}$ regularization of \citet{louizos2018learning}, a continuous approximation to an $\ell_0$ penalty that encourages sparsity by penalizing nonzero values. However, it does not explicitly penalize “binary violation” (mask values near 0.5). In PGB-CT, our penalty combines an $\ell_1$ term with a term proportional to $\bm(\mathbf{1}-\bm)$, which directly penalizes values away from $\{0,1\}$ and thus encourages masks to be close to binary throughout optimization.
We later demonstrate the effectiveness of our design in~\algoname{}.

\setlength{\belowrulesep}{2pt}  % adjust as needed
\renewcommand{\arraystretch}{0.75}
\begin{table*}[th]
    \centering
    \begin{adjustbox}{max width=\textwidth}
    \begin{tabular}{ll|*{4}{c}|c|*{4}{c}|c}
        \toprule
        \multicolumn{2}{l|}{%
            \multirow{4}{*}{%
                \makecell[lt]{%
                \diagbox[innerwidth=4.5cm, height=1cm]{Model}{\vspace{-0.2 cm}Dataset}%
                }
            }
        }
          & \multicolumn{5}{c|}{WinoGender} 
          & \multicolumn{5}{c}{WinoBias} \\
        \cmidrule(lr){3-7}  \cmidrule(lr){8-12} 
         & & \multicolumn{4}{c|}{Sparsity~~$S/N$}& Time & \multicolumn{4}{c|}{Sparsity~~ $S/N$}& Time\\ 
        \cmidrule(lr){3-6}  \cmidrule(lr){7-7} \cmidrule(lr){8-11} \cmidrule(lr){12-12}
        & & 10\% & 20\% & 30\% & 40\% & minutes  & 10\% & 20\% & 30\% & 40\% & minutes \\
        \midrule
        % --- DistilGPT2 ---
        \multirow{4}{*}{DistilGPT2} 
          & random   & 0.0082 & 0.0093 & 0.0091 & 0.0100 & 2.23 & 0.0128 & 0.0146 & 0.0146 & 0.0160 & 6.55  \\
          & top-$k$    & 0.0100 & 0.0117 & 0.0123 & 0.0128 & \textbf{0.18} & 0.0157 & 0.0171 & 0.0178 & 0.0181 & \textbf{0.47} \\
          & greedy   & \textbf{0.0103} & \textbf{0.0122} & \textbf{0.0134} & \textbf{0.0140} & 3.57 & \textbf{0.0159} & \textbf{0.0176} & \textbf{0.0186} & \textbf{0.0192} & 10.62 \\
          & \algoname & \textbf{0.0103} & \underline{0.0118} & \underline{0.0130} & \underline{0.0137} & \underline{0.43} & \textbf{0.0159} & \underline{0.0174} & \underline{0.0181} & \underline{0.0190} &  \underline{1.44} \\
        \midrule
        % --- GPT2 ---
        \multirow{4}{*}{GPT2-small} 
          & random   & 0.027 & 0.032 & 0.036 & 0.037 & 3.54 & 0.080 & 0.096 & 0.100 & 0.103 & 10.25 \\
          & top-$k$    & 0.037 & 0.041 & 0.042 & 0.044 & \textbf{0.50} & 0.108 & 0.111 & 0.113 & 0.114 & \textbf{1.49} \\
          & greedy   & \textbf{0.040} & \textbf{0.045} & \textbf{0.049} & \textbf{0.050} & 23.03 & \textbf{0.111} & \textbf{0.116} & \textbf{0.119} & \textbf{0.120} & 68.42 \\
          & \algoname & \underline{0.039} & \underline{0.044} & \underline{0.048} & \underline{0.049} & \underline{0.79} & \underline{0.109} & \underline{0.115} & \underline{0.117} & \underline{0.119} & \underline{2.61} \\
        \midrule
        % --- GPT2-medium ---
        \multirow{4}{*}{GPT2-medium} 
          & random   & 0.138 & 0.153 & 0.165 & 0.173 & 7.23 & 0.190 & 0.247 & 0.305 & 0.336 & 21.45 \\
          & top-$k$    & 0.191 & 0.201 & 0.203 & 0.205 & \underline{2.76} & 0.374 & 0.378 & 0.389 & 0.388 & \underline{8.18} \\
          & greedy   & \textbf{0.208 }& \textbf{0.224} & \textbf{0.232} & \textbf{0.237} & 357.28 & \textbf{0.391} & \textbf{0.406} & \textbf{0.415} & \textbf{0.420} & 1001.50 \\
          & \algoname & \underline{0.203} & \underline{0.218} & \underline{0.227} & \underline{0.233} & \textbf{1.56} & \underline{0.381} & \underline{0.394} & \underline{0.401} & \underline{0.404} & \textbf{5.32} \\
        \midrule
        % --- GPT2-large ---
        \multirow{3}{*}{GPT2-large} 
          & random   & 0.085 & 0.099 & 0.111 & 0.128 & 13.43 & 0.164 & 0.216 & 0.262 & 0.297 & 36.8 \\
          & top-$k$    & \underline{0.146} & \underline{0.159} & \underline{0.165} & \underline{0.168} & \underline{9.46} & \underline{0.353} & \underline{0.357} & \underline{0.361} & \underline{0.364} & \underline{26.07} \\
          & \algoname  & \textbf{0.152} &\textbf{0.171} & \textbf{0.184} & \textbf{0.192} & \textbf{2.43} & \textbf{0.356} & \textbf{0.370} & \textbf{0.375} & \textbf{0.383} & \textbf{8.10} \\
        \midrule
        % --- GPT2-xl ---
        \multirow{3}{*}{GPT2-xl} 
          & random   & 0.047 & 0.078 & 0.094 & 0.098 & \underline{24.83} & 0.221 & 0.393 & 0.442 & 0.460 & 64.98 \\
          & top-$k$    & \underline{0.155} & \underline{0.175} & \underline{0.186} & \underline{0.188} & 29.38 & \underline{0.523} & \underline{0.531} & \underline{0.536} & \underline{0.539} & \underline{62.85} \\
          & \algoname & \textbf{0.159} & \textbf{0.190} & \textbf{0.205} & \textbf{0.212} & \textbf{3.45} & \textbf{0.546} & \textbf{0.563} & \textbf{0.567} & \textbf{0.576} & \textbf{11.32} \\
        \midrule
        % --- Qwen3-1.7B ---
         \multirow{3}{*}{Qwen3-1.7B}
          & random   & 0.032 & 0.035 & 0.041 & 0.038 & 29.36 & 0.064 & 0.073 & 0.077 & 0.085 & 75.38 \\
          & top-$k$    & \underline{0.041} & \underline{0.042} & \underline{0.045} & \underline{0.044} & \underline{13.01} & \underline{0.089} & \underline{0.090} & \underline{0.091} &  \underline{0.092} &  \underline{31.97} \\
          & \algoname &  \textbf{0.042} & \textbf{0.046} & \textbf{0.048} & \textbf{0.050} & \textbf{3.21} & \textbf{0.090} & \textbf{0.094} & \textbf{0.096} & \textbf{0.097}& \textbf{9.39} \\
        \midrule
        % --- Llama-3.2-1B ---
        \multirow{3}{*}{Llama-3.2-1B} 
          & random   & 0.326 & 0.383 & 0.387 & \underline{0.418} & 20.71 & \underline{0.322} & 0.326 & 0.325 & 0.351 & 51.16 \\
          & top-$k$   & \textbf{0.450} & \underline{0.416} & \underline{0.400} & 0.366 & \underline{10.31} & 0.279 & \underline{0.358} & \underline{0.386} & \underline{0.415} & \underline{26.57} \\
          & \algoname & \underline{0.366} & \textbf{0.444} & \textbf{0.477} & \textbf{2.449}& \textbf{1.90} & \textbf{0.752} & \textbf{1.007} & \textbf{1.766} & \textbf{2.747} & \textbf{5.82} \\
        \bottomrule
    \end{tabular}
    \end{adjustbox}
    \caption{Averaged metric $\ell(\mcD,\bm)$ and execution time on the WinoGender and WinoBias datasets for selecting attention heads across different sparsity levels, algorithms, and LLMs. The best and second-best results are highlighted in \textbf{bold} and \underline{underlined}, respectively.}
    \label{tab:results_attention}
\end{table*}
\renewcommand{\arraystretch}{1.0}

\begin{figure*}[tbp]
    \centering
    \vspace{-0.06 in} 
    \includegraphics[width=0.42\textwidth]{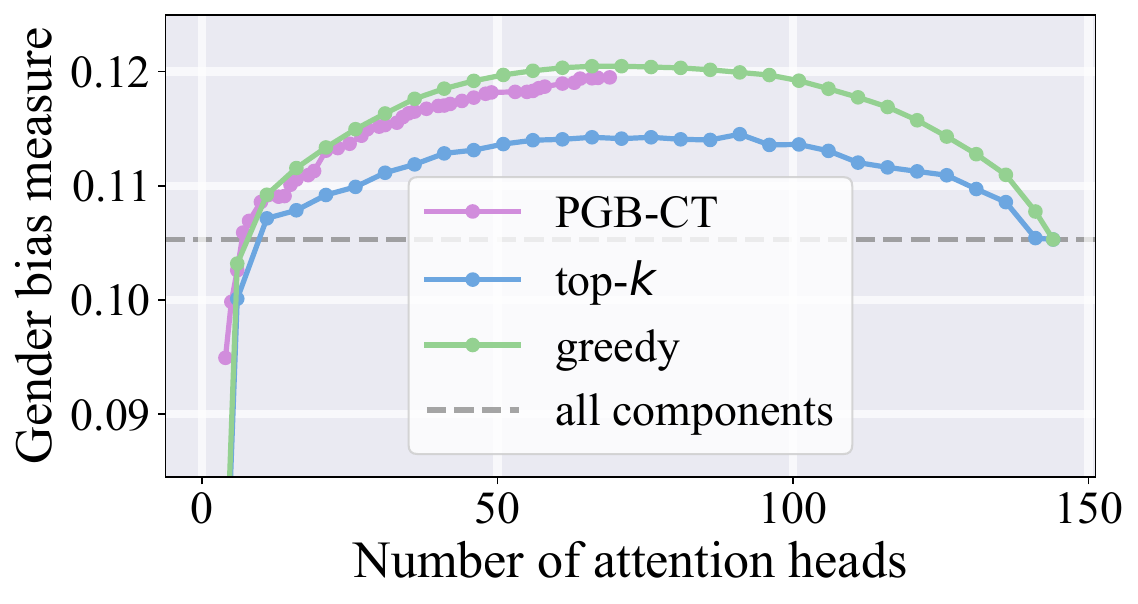}\hfill
    \includegraphics[width=0.46\textwidth]{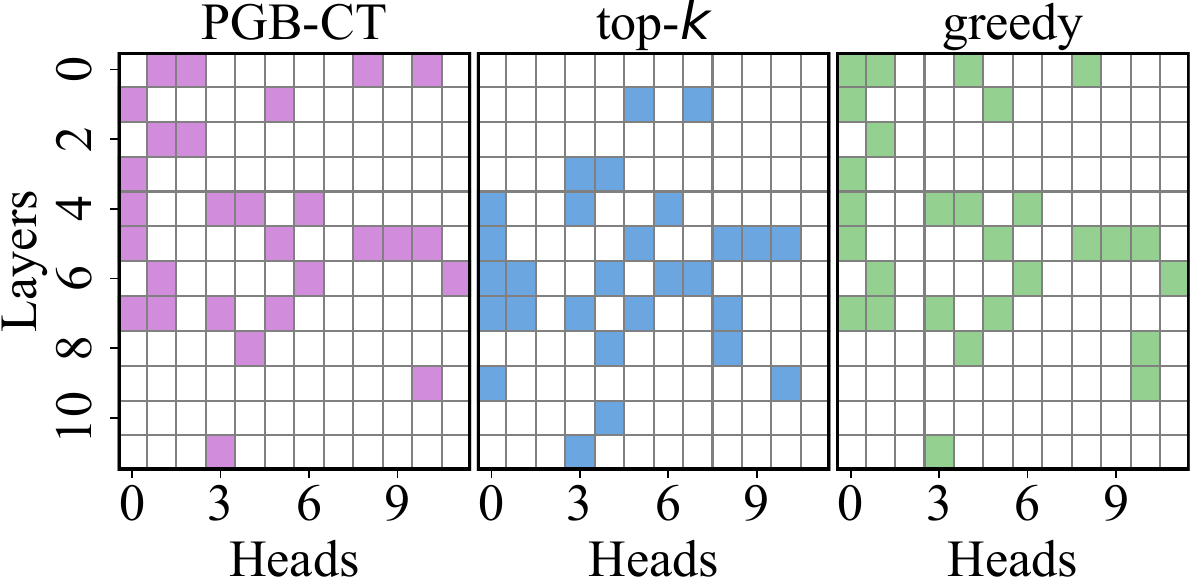} 
    \caption{Results of attention heads from GPT2-small on the WinoBias dataset. \textit{Left:} Gender bias measure vs. number of attention heads. \textit{Right:} Selected attention heads with $20\%$ sparsity.}
    \label{fig:results_attention} 
    \vspace{-0.1 in}
\end{figure*}

\section{Experimental Results}
\label{sec:experiment}
We present our experimental setup and compare \algoname{} with several baselines. The detailed description of datasets, baselines, and parameter selections is relegated to Appendix~\ref{sec:experimentdetails}, and the additional experiments with further discussions are relegated to Appendix~\ref{sec:additionalexperiments}. 
\subsection{Experiment Setup}
\noindent\textbf{Models.} Following the causal tracing study in~\citet{vig2020investigating}, we focus on the GPT2 family~\citep{radford2019language} and consider the following model sizes \citep{wolf2019huggingface}: small, medium, large, and xl, and a distilled GPT model trained by~\citet{sanh2019distilbert}. In addition, we report results for more new architectures with a comparable number of parameters: Qwen3-1.7B~\citep{yang2025qwen3} and Llama3.2-1B~\citep{llama32}. We also use LLaMA-7B and LLaMA-13B~\cite{touvron2023llama} to compare with DCM in their setting.

\begin{figure*}[tbp]
    \centering    \includegraphics[width=0.42\textwidth]{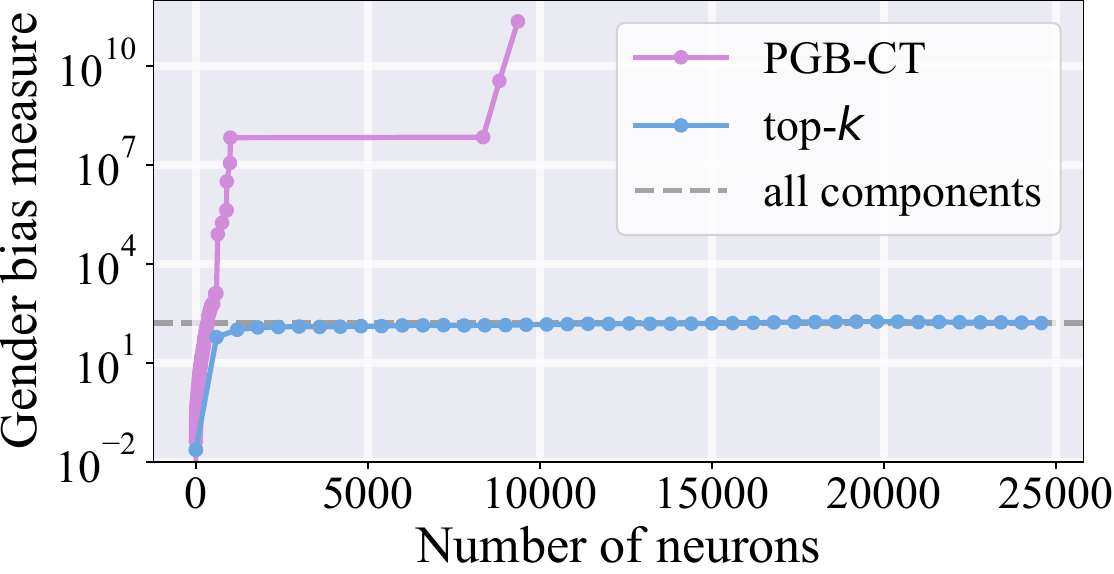}\hspace{0.5 in} \includegraphics[width=0.48\textwidth]{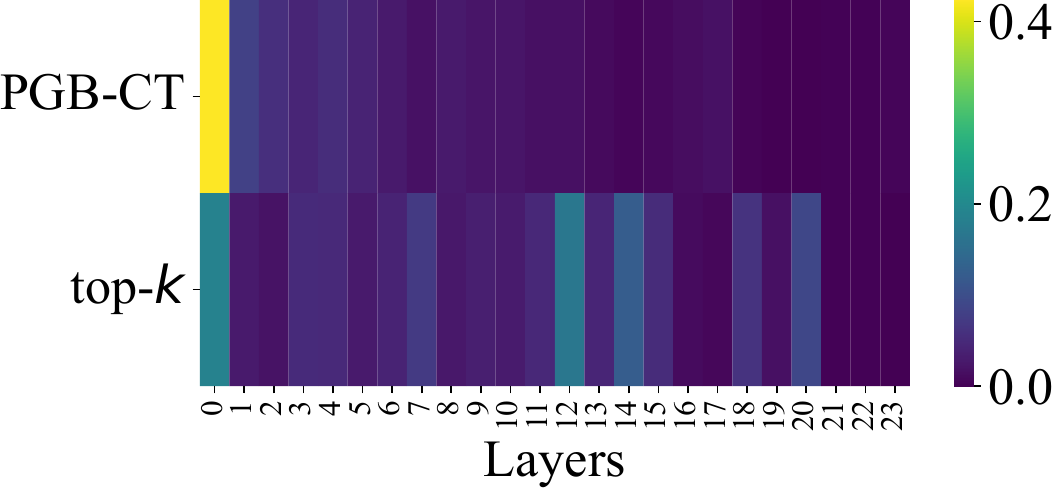} 
    \caption{Results of selecting MLP neurons on the Professions dataset with GPT2-medium. \textit{Left:} Gender bias  vs.~number of neurons. \textit{Right:} Fraction of selected neurons in each layer with 5\% sparsity.}
    \label{fig:results_MLP}
\end{figure*}

\vspace{0.04 in}
\noindent\textbf{Datasets.}
Following previous studies~\cite{vig2020investigating,meng2022locating}, we evaluate our algorithm on four datasets: two with attention weights as the candidate components and two with MLP neurons. We include further details about each dataset in Appendix~\ref{sec:dataset}. We also include experiments on the Variable Binding Desiderata (VBD) dataset~\cite{davies2023discovering} to demonstrate the effectiveness of our algorithm in simultaneously identifying attention heads and MLPs.

\vspace{0.04 in}
\noindent\textbf{Baselines.} \noindent We compare \algoname{} against three baseline algorithms (details in Appendix~\ref{sec:baseline}.): \emph{top-$k$ selection} and \emph{greedy selection} algorithms, used by \citet{vig2020investigating} to generalize single-component $S=1$ algorithms to the same multi-component problem in~\eqref{eq:problem}. We also include \emph{random selection}, which identifies the best one among randomly sampled $1000$ candidates of sparsity $S$, to serve as a baseline for execution time. We additionally evaluate the Desiderata-based Component Masking (DCM) algorithm~\cite{davies2023discovering}, but we found that it consistently failed in practice; further details are provided in Appendix~\ref{sec:DCM}. All experiments are run on a single NVIDIA A6000 GPU, and we report the execution time for each algorithm. All algorithms have well-defined stopping criteria, except for the greedy algorithm, for which execution time is measured when the algorithm achieves $40\%$ sparsity. In some settings, the runtimes of the top-$k$ and greedy algorithms become excessively large. We omit these results when the runtimes exceed 2 days.

%%%%%%%%%%%%%%%%%%%%%%%%%%%%%%%%%
% results attention heads
\subsection{Numerical Results on Attention Heads}
% \vspace{0.02 in}
\noindent\textbf{Attention Heads.} 
Table~\ref{tab:results_attention} summarizes the results on the WinoGender and WinoBias datasets when intervening on attention heads at various sparsity levels. To provide a more meaningful and comparable measure of “what fraction of the model” is being used across architectures and model sizes, we use the sparsity level $S/N$, where $N$ is the total number of components. We report the average performance metric $\ell(\mcD,\bm)$ and the execution time in minutes. The best and second-best results are highlighted in \textbf{bold} and \underline{underlined}, respectively. Figure~\ref{fig:results_attention} shows how the performance metric varies with the number of attention heads selected by three algorithms on GPT2-small (left) and highlights the heads selected at $20\%$ sparsity (right). Here, the horizontal ``all component'' line shows the result of intervening on all components. Our main observations from these results are as follows:
\begin{itemize}[leftmargin=0.15in, itemsep=0.01in, topsep=0.02in]
    \item \textbf{Averaged metric.} Across all models and sparsity levels, \algoname{} yields higher $\ell(\mcD,\bm)$ values compared to the top-$k$ algorithm and is comparable to the greedy algorithm, demonstrating its ability to find the influential attention heads. Furthermore, as the model size grows, the gap between the top-$k$ and greedy or our method increases, showing that top-$k$ does not scale well.\vspace{-.05 in}
    \item \textbf{Execution time.} \algoname{} significantly reduces the execution time compared to both the greedy and random selection algorithms and becomes faster than top-$k$ as the model size increases (larger than GPT2-medium). This performance advantage holds because the top-$k$ algorithm requires a forward pass for every component across all samples, and the greedy algorithm is even worse in this sense. In contrast, \algoname's execution time does not explicitly depend on the size of the search space; its only additional cost arises from the longer forward and backward passes inherent in larger LLMs. However, it still uses fewer samples than the random algorithm.\vspace{-.1 in}
    \item \textbf{Components selection.}  The attention heads selected by \algoname{} have a higher Jaccard similarity to those chosen by the greedy algorithm (0.64) than to those selected by the top‑$k$ algorithm (0.44). Also, both \algoname{} and the greedy algorithm favor the first layer while avoiding layer 10, and differ only in 3 components.
\end{itemize}

\begin{figure*}[tbp]
    \centering    \includegraphics[width=0.4\textwidth]{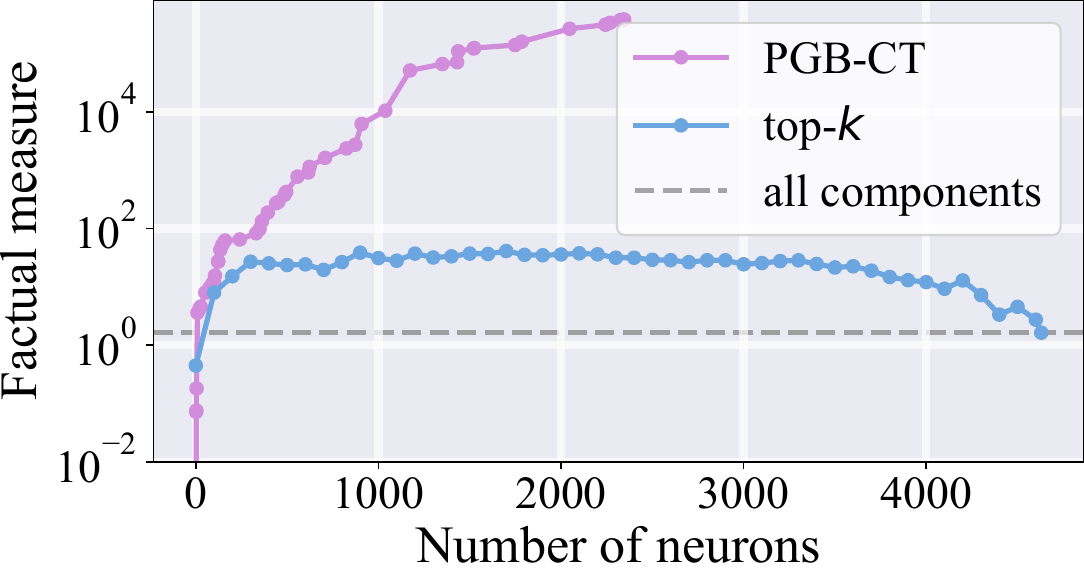}\hfill
    \includegraphics[width=0.5\textwidth]{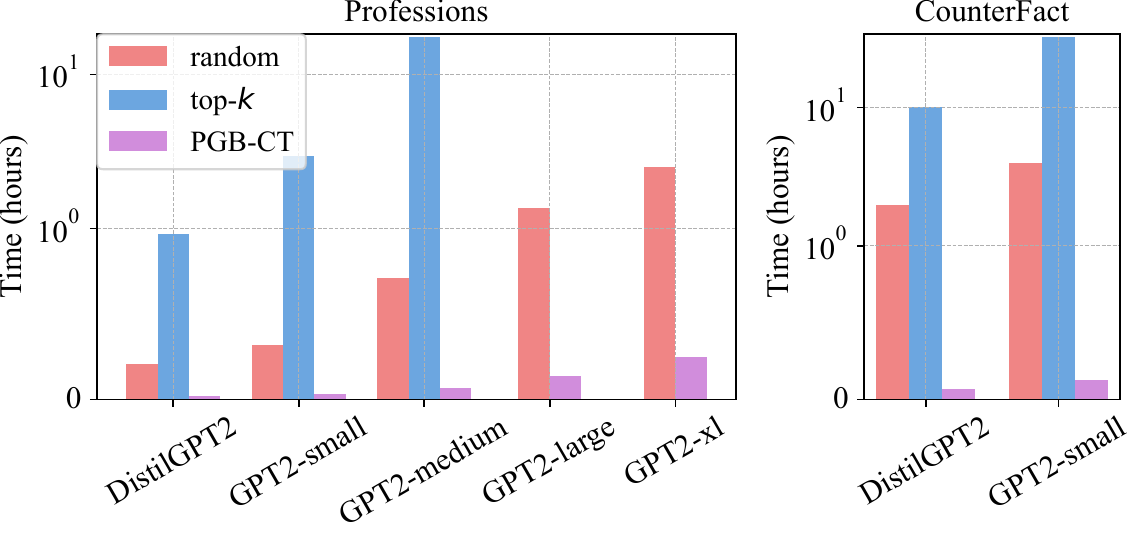} 
    \caption{\textit{Left:} Results of Factual locating measure vs.~number of MLP neurons on the CounterFact dataset with distilGPT2.   \textit{Right:} Execution time for different algorithms on Professions and CounterFact datasets.} 
    \label{fig:results_factual}  
\end{figure*}

\vspace{0.04 in}
\noindent\textbf{MLP Neurons.}
Figures~\ref{fig:results_MLP} and \ref{fig:results_factual} extend these findings to interventions on MLP neurons using the Professions and CounterFact datasets. Our key findings are as follows.
\begin{itemize}[leftmargin=0.15in, itemsep=0.01in, topsep=0.02in]
    \item \textbf{Averaged metric.} \algoname{} achieves significantly higher metric values $\ell(\mcD,\bm)$ compared to both random and top-$k$ baseline algorithms across different models and sparsity levels. Furthermore, under \algoname, the metric increases almost exponentially after top-$k$ reaches a plateau. (Figure~\ref{fig:results_MLP} left). Complete comparisons are provided in Table~\ref{fig:results_MLP} and Table~\ref{fig:results_factual} in Appendix~\ref{sec:additionalneuron}. This demonstrates a superior ability to identify critical sets of MLP neurons. \vspace{-.1 in}
    \item \textbf{Execution time.} Figure~\ref{fig:results_factual} (right) shows that the \algoname{} algorithm reduces the execution time from several hours in the baseline algorithms to only a few minutes, thereby enabling efficient exploration in large LLMs. \vspace{-.05 in}
    \item \textbf{Components selection.} Figure~\ref{fig:results_MLP} (right) shows the difference in selected neurons for top-$k$ and \algoname{}. \algoname{} tends to select more in the first layer, while top-$k$ relies more on later layers. 
\end{itemize}

\noindent\textbf{Attention Heads with MLP Neurons.}
\begin{figure}[!t]
    \centering
\includegraphics[width=0.8\columnwidth]{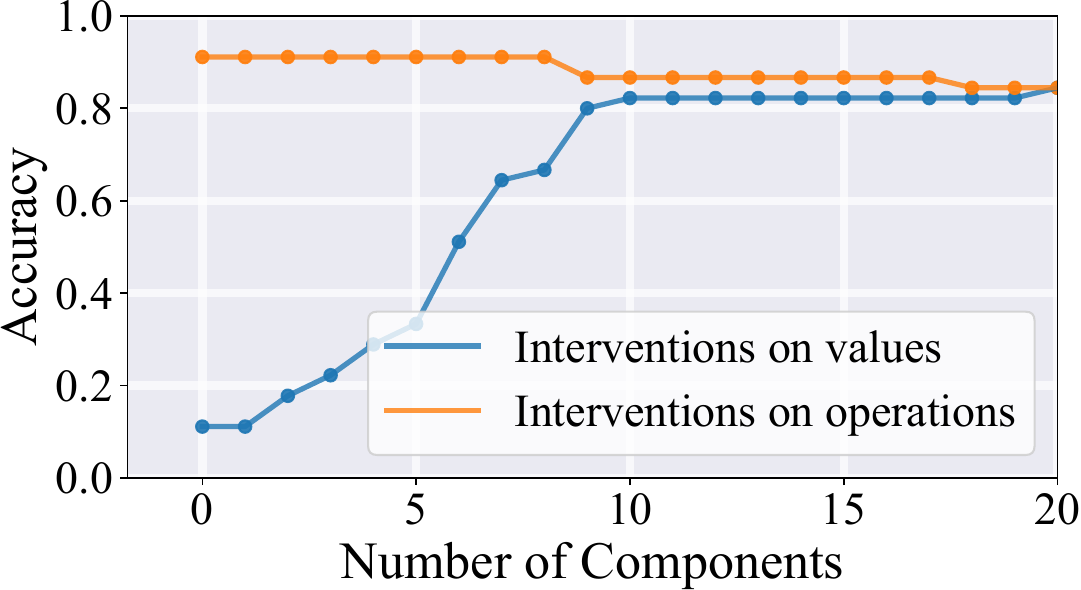}
    \caption{Results on the VBD dataset with LLaMA-13B, under interventions on the values or operations in the VBD equations.}
    \label{fig:VBD}
    \vspace{-0.1 in}
\end{figure}
We find that a straightforward joint analysis is challenging because when all attention heads and MLP neurons are treated as a single pool of components, the model tends to select MLPs almost exclusively, even at extremely low sparsity. This effect arises due to two factors:
\begin{itemize}[leftmargin=0.15in, itemsep=0.03in, topsep=0.03in]
            \item The number of MLP neurons is far larger than the number of attention heads (e.g., 24,576 versus 384 in GPT2-medium).
            \item The combined effect of MLP neurons is much stronger than that of attention heads (i.e., MLP interventions dominate). 
        \end{itemize}
        To gain better insight and move beyond the dominance of MLP neurons, we perform group interventions as in \citet{davies2023discovering}. Specifically, we group the neurons in each layer into a single component. This results in a number of neuron groups that is comparable to the number of attention heads. In this setting, the algorithm can effectively find the combination of attention heads and MLP blocks. 
 Following this insight, we conducted experiments that jointly analyzed attention heads and \emph{whole MLP blocks} and scaled them up to LLaMA-13B, which contains 1,400 attention heads and 40 blocks. The results are presented in Figure~\ref{fig:VBD} (see a similar result on LLaMA-7B in Appendix~\ref{sec:vadadd}). Importantly, our algorithm successfully identifies effective combinations of attention heads and MLPs. 

When the sparsity level is $S=10$, we find the following combination: Attention Heads 11.11, 12.7, 15.11, 15.25, 16.1,  18.18, 19.25, 21.13, and MLP blocks 5, 6.

\section{Discussion and Conclusion}
 In this paper, we presented a unified framework for causal tracing across multiple components within LLMs. To tackle the combinatorial search space, we proposed a method that reduces the computational burden while achieving high metric values, providing empirical insights into how key model components collectively influence performance. Taken together, our results illuminate significant patterns in how LLMs process information, setting a new foundation for future research.

\clearpage
\section*{Limitations}
While our approach efficiently handles a wide range of scenarios and metrics, it currently assumes that a single fixed target metric or objective is specified in advance for the finite dataset. In addition, our algorithm requires hyperparameter tuning to ensure the desired sparsity. Finally, due to the limitation of computational resources and the inefficiency of baselines, our experiments focus primarily on English datasets and the GPT architecture and other architectures like Llama and Qwen of similar size; further research is needed to validate the generalizability of our methods across different languages, model families, parameter sizes, and tasks with highly specialized or domain-specific objectives. The \algoname{} algorithm uses gradient-descent-based modules, and the optimization method can be replaced with other methods of interest in a modular way. Gradient descent has many hyperparameters, such as batch size, learning rate schedule, and optimizer choice, which typically require tuning. Although gradient descent has been highly successful in practice, it does not guarantee convergence to a global minimum.

\section*{Acknowledgments}
This work was supported in part by the Rensselaer-IBM Future of Computing Research Collaboration (FCRC) and the US National Science Foundation Center for Research toward Advancing Financial Technologies.

\bibliography{custom}

\appendix

\section{Difference from Circuit Analysis}
Most circuit analysis methods (e.g., \citet{chughtai2023toy,wang2023interpretability, conmy2023towards}) are designed for \emph{mechanism discovery on small, hand-crafted tasks}, in which the goal is to recover a specific circuit and evaluate overlap with a ground-truth subgraph. In contrast, our goal is to optimize a \emph{dataset-level causal objective} (e.g., a bias or factuality metric), and there is no ground-truth circuit to compare against.

Furthermore, the recent studies on efficient circuit finding (e.g., \cite{syed2023attribution,hanna2024have}) still typically operate at a much finer granularity (e.g., edges or neuron-position pairs) and focus on metrics such as normalized faithfulness of the recovered circuit. Applying these methods as baselines in our setting would require substantial adaptation in the algorithm and would not yield a clean comparison to our metric.

\section{Pseudocode}
\label{sec:code}
Algorithm~\ref{alg:CT} presents the pseudocode for the proposed \algoname. This pseudocode outlines the key steps and training process of our approach.

\begin{algorithm}[th]
\caption{\algoname{} algorithm}
\label{alg:CT}
\begin{algorithmic}[1]
\STATE \textbf{Inputs:} Language model $f$, dataset $\mcD$, target metric $\ell(\cdot)$, time steps $T$, sparsity level $S$, scheduler for $\lambda_1$ and $\lambda_2$
\STATE \textbf{Initialization:} Initialize weight vector $\bm_0$, \STATE set $t=0$
\WHILE{$t \leq T$}
    \STATE Sample a mini-batch $\mcB_t =\{\bs_i, \bx_i, \by_i\}_{i=1}^{B}$ from the dataset $\mcD$
    \STATE Get hidden states $\{\bh_i , \in[B]\}$ using the prompts $\{\bs_i, i\in[B]\}$ according to \eqref{eq:fwd}
    \STATE Get hidden states $\bh'$ using counterfactual prompts $\{\bs_i', i\in[B]\}$ according to \eqref{eq:fwd}
    \STATE Perform an interventional run to obtain $\bar\bh$ according to \eqref{eq:mixfwd}
    \STATE Calculate the loss function $\mcL(\mcD, \bm_t)$ according to \eqref{eq:loss}
    \STATE Update $\bm_{t+1}$ according to \eqref{eq:gradiendescent}
    \STATE Form the discretized set $\mcH_{t+1}$ using \eqref{eq:discrete} 
    \STATE Break if $|\mcH_{t+1}|\leq S$
\ENDWHILE
\end{algorithmic}
\end{algorithm}

\begin{table*}[hbt]
    \centering
    {
    \small
    \begin{tabular}{lccccc}
        \toprule
        \textbf{Model} & \textbf{Layers} & \textbf{Heads per layer} & \textbf{ Neurons per layer} & \textbf{Total heads} & \textbf{Total neurons} \\
        \midrule
        DistilGPT2 & 6 & 12 & 768 & 72 & 4608 \\
        GPT2-small & 12 & 12 & 768 & 144 & 9216 \\
        GPT2-medium & 24 & 16 & 1024 & 384 & 24576 \\
        GPT2-large & 36 & 20 & 1280 & 720 & 46080 \\
        GPT2-xl & 48 & 25 & 1600 & 1200 & 76800 \\
        Qwen3-1.7B & 28 & 16 & 2048 & 448 & 57344 \\
        Llama3.2-1B & 16 & 32 & 2048 & 512 & 32768 \\
        LLaMA-7B & 32 & 32 & 1 block & 1024 & 32 blocks\\
        LLaMA-13B & 40 & 40 & 1 block & 1600 & 40 blocks\\
        \bottomrule
    \end{tabular}
    \caption{Specifications of LLMs on the number of attention heads and MLP neurons.}
    \label{tab:gpt2_models}
    }
\end{table*}

\section{Additional experiment details}
\label{sec:experimentdetails}
\subsection{Model Dimension}
\label{sec:model dimension}
Table~\ref{tab:gpt2_models} shows the configurations of various GPT and other models and specifies the corresponding number of components $N$ used for attention heads and MLP neuron experiments. We notice that the Qwen model uses Grouped Query Attention (GQA), so we use the number of attention heads for $Q$ as the number of attention heads in the model.

We note that GPT models are under the modified MIT License; DistilGPT2 and Qwen3 are under Apache 2.0; Llama 3.2 is under Llama 3.2 Community License; Llama is research-only, and there is no non-commercial license.

\setlength{\tabcolsep}{3pt}
\begin{table*}[htbp]    \centering
    {\fontsize{8pt}{9pt}\selectfont 
    \begin{tabular}{l|ll l}
        \toprule
        \textbf{Dataset Name} & \textbf{Example Entry} &\textbf{Example Output} & \textbf{\# Samples} \\
        \midrule
        \multirow{2}{*}{WinoGender} & $\bs$: The technician told the customer that \color{BrickRed}{she} & $\bx$: could pay with cash. & \multirow{2}{*}{44} \\
        & $\bs'$: The technician told the customer that \color{BrickRed}{he} & $\by$: had completed the repair. & \\
        \midrule
        \multirow{2}{*}{WinoBias} & $\bs$: The chief gave the housekeeper a tip because \color{BrickRed}{she} & $\bx$: was helpful. &\multirow{2}{*}{160/130}  \\
          & $\bs'$: The chief gave the housekeeper a tip because \color{BrickRed}{he} & $\by$: was satisfied. & \\
        \midrule
        \multirow{2}{*}{Professions} & $\bs$: The {\color{BrickRed}{nurse}} said that & $\bx$: she &\multirow{2}{*}{2873}  \\
        & $\bs'$: The {\color{BrickRed}{man}} said that & $\by$:  he &  \\
        \midrule
        \multirow{2}{*}{CounterFact} & $\bs$: The Space Needle lies in downtown & $\bx$: Seattle &\multirow{2}{*}{200}  \\
        & -- & -- &  \\
        \midrule
        \multirow{3}{*}{VBD} & $\bs$: x = {\color{BrickRed}{5}}, y = 2, x {\color{BrickRed}{+}} y = & -- &\multirow{3}{*}{50/50}  \\
        & $\bs_1'$: x = {\color{BrickRed}{4}}, y = 2, x + y = & $\by_1$: $6$:     \\
         & $\bs_2'$: x = 5, y = 2, x {\color{BrickRed}{-}} y = & $\by_2$: $7$ &  \\
         \bottomrule
    \end{tabular}
    }
    \caption{Examples and Sample Counts for different datasets}
    \label{tab:datasets}
\end{table*}
\setlength{\tabcolsep}{6pt}

\subsection{Dataset Details}
\label{sec:dataset}

Table~\ref{tab:datasets} presents examples and the corresponding sample counts for each dataset. Additional details are summarized below.
\begin{itemize}
   \item \textbf{WinoGender dataset.} \cite{rudinger2018gender}, under MIT license. 
    Originating from the Winograd Schema Challenge, this dataset features sentences designed to test a model's ability to resolve ambiguous pronouns. For example, in the sentence “The trophy doesn’t fit in the suitcase because it is too small,” the pronoun “it” correctly refers to “suitcase.” The dataset extends this idea by incorporating sentences where gendered pronouns (e.g., “he,” “she”) are linked to occupations or roles. Consider these examples: “The nurse finished their shift, and he went home” and “The doctor prescribed the medication, and she ensured the patient took it.” This setup helps assess whether a model’s coreference resolution is influenced by pronoun gender. We follow the preprocessing protocol from~\citet{vig2020investigating}: omitting filtering and using the \emph{bergsma} start, yielding a total of 44 samples.
 \item \textbf{WinoBias.} \cite{zhao2018gender}, under MIT license.  
    Like WinoGender, the WinoBias dataset examines gender bias in coreference resolution. However, it shifts the focus from mere pronoun probability predictions to verifying that models correctly associate gendered pronouns with stereotypically linked professions. For instance, when presented with a sentence that mentions a typically male-associated profession (e.g., “doctor”) alongside a typically female-associated one (e.g., “nurse”), the model should assign “he” or “she” appropriately without bias. After applying filtering criteria, the development split contains 160 samples and the test split contains 130 samples. Our experiments use the development set, as its stability has been noted in~\citet{vig2020investigating}.
    \item \textbf{Professions dataset.} \cite{vig2020investigating}, under MIT License.
    This dataset is built on an expanded set of templates from~\citet{lu2020gender}, combined with a list of professions from~\citet{bolukbasi2016man}, and is intended for neuron-level intervention analysis. The templates—formatted like “The [occupation] [verb] because”—were filtered following the approach in~\cite{vig2020investigating} to avoid issues with sub-word tokenization. The result is a collection of 17 templates and 169 professions, which together form 2,873 examples.
     \item \textbf{CounterFact.} Following \cite{meng2022locating},  
    this dataset comprises factual assertions represented as tuples $(s, r, o)$, where $s$ is a subject entity (e.g., \textit{Paris}), $r$ denotes a binary relation (e.g., \textit{is located in}), and $o$ is the corresponding object (e.g., \textit{France}). Each tuple encapsulates a specific fact about the world. To ensure that baseline models can run efficiently, we randomly subsample 200 examples from the full set to obtain a suitable sized dataset for evaluation.
    \item \textbf{Variable Binding Desiderata.} \cite{davies2023discovering}, under MIT License. The Variable Binding Desiderata (VBD) dataset consists of 3-tuples $(\bs, \bs', \by)$, where $\bs$ is the original prompt, $\bs'$ is the 
    counterfactual prompt, and $\by$ is the target. Each prompt has the form ``x = $a$, y = $b$, x \{op\} y ='', where $a$ and $b$ are random numbers  and \{op\} is one of the operators in $\{+,-,\times,\div\}$. The dataset encodes two competing desiderata: value dependence and operation invariance. Under value dependence, intervention should change the output to match that of the counterfactual; under operation invariance, intervention should leave the output unchanged. Hence, the metric will be the accuracy of predicting the next token.
\end{itemize}

\subsection{Baselines}
\label{sec:baseline}
To assess the performance of our algorithm, we compare \algoname{} against the following baselines.
\begin{itemize}[leftmargin=0.15in, itemsep=0.01in, topsep=0.08in]
\item \textbf{random.} This approach first randomly selects $1000$ of such combinations (i.e., $1000$ sets each with $S$ components), evaluates the entire set together, and identifies the best set among the $M$ chosen ones.
\item \textbf{top-$k$.} This approach individually examines each of the $N$ components, evaluates the metric $\ell(\mcD,\bm)$ after intervening on each individually, and then sorts the resulting scores to pick the best
$S$ components.
\item \textbf{greedy.} Components are added iteratively by selecting, at each step, the component that has the maximum metric $\ell(\mcD,\bm)$ after intervening on the selected set plus each individually remaining component. The process ends when the sparsity requirement is met.
\end{itemize}

We also considered the Desiderata-based Component Masking (DCM) algorithm proposed by \citet{davies2023discovering}, which adopts a similar idea of soft masking and gradient descent. However, as we show later in Appendix~\ref{sec:DCM}, its loss function and penalty function design lead to suboptimal results.

\paragraph{Complexity of Baselines.} We compare methods by the number of \emph{interventional runs} (the dominant cost). 
\begin{itemize}
  \item \textbf{Random:} fixed at $1000$ interventional runs.
  \item \textbf{top-$k$:} $N$ interventional runs.
  \item \textbf{greedy (select $S$ components):} To add each component, it needs to run through all remaining components, and hence has complexity 
  \begin{equation*}
    \sum_{s=1}^{S} (N - s) \;=\; NS - \frac{S(S+1)}{2} \;=\; \Theta(NS) \ .
  \end{equation*}
\end{itemize}
When the model is small ($N < 1000$), \emph{top-$k$} is typically the fastest; for larger models ($N > 1000$), \emph{random} becomes more efficient than top-$k$ and greedy due to its fixed budget. We also note that \texttt{DCM} has a sample complexity similar to \algoname{}, as both are gradient-based methods, which are fastest in large models.

\begin{figure}[htbp]
    \centering
    \includegraphics[width=0.95 \columnwidth]{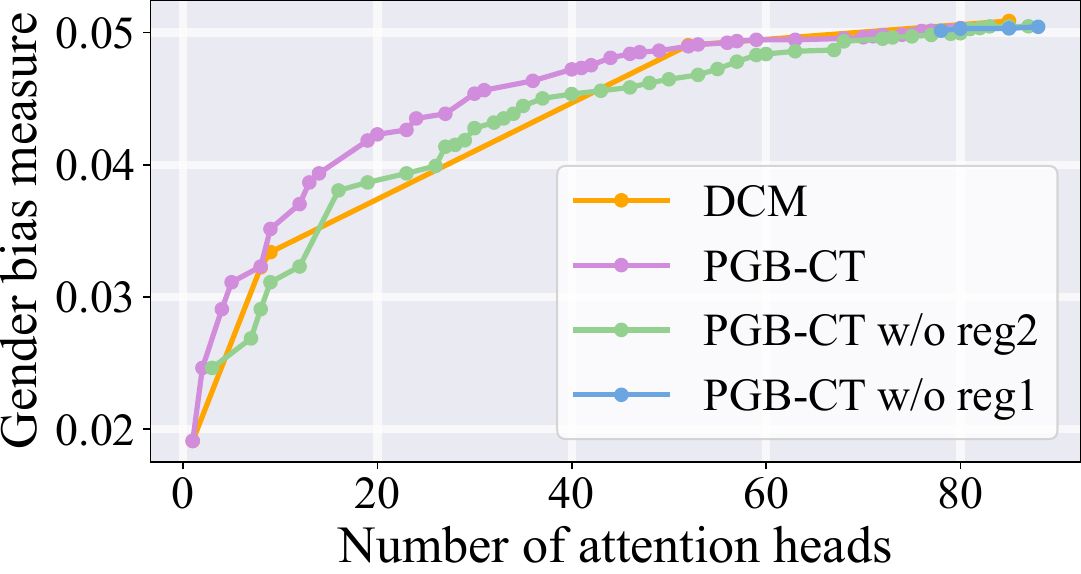}\hfill
    \caption{Results of an ablation study when removing the first and second penalty term on WinoGender dataset with GPT2-small.}
    \label{fig:Ablation study}
\end{figure}

\begin{figure*}[!htbp]
    \centering    \includegraphics[width=0.44\textwidth]{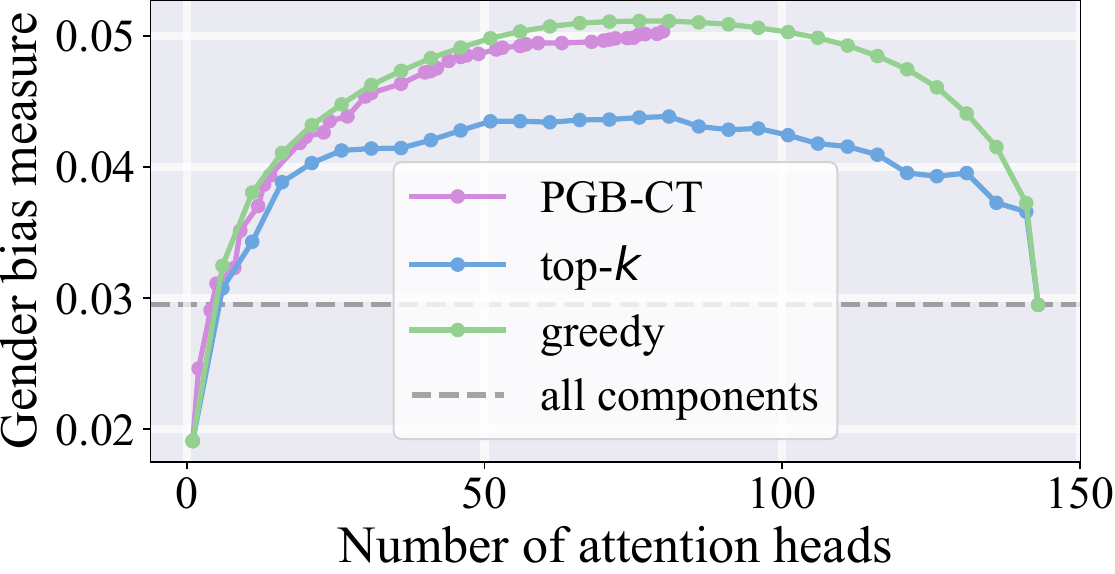}\hfill
\includegraphics[width=0.5\textwidth]{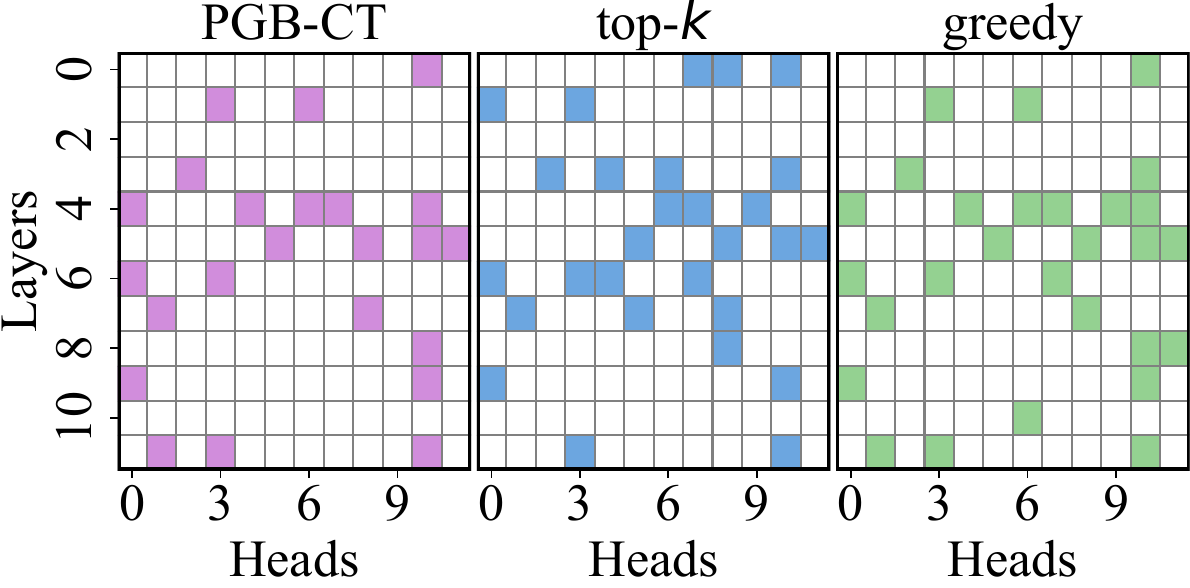}
    \caption{Results of selecting attention heads from GPT2-small on the WinoGender dataset. \textit{Left:} Gender bias vs. number of attention heads. \textit{Right:} Selected attention heads with $20\%$ sparsity.}
    \label{fig:attention_addition_2}
\end{figure*}

\begin{figure*}[!htbp]
    \centering    \includegraphics[width=0.44\textwidth]{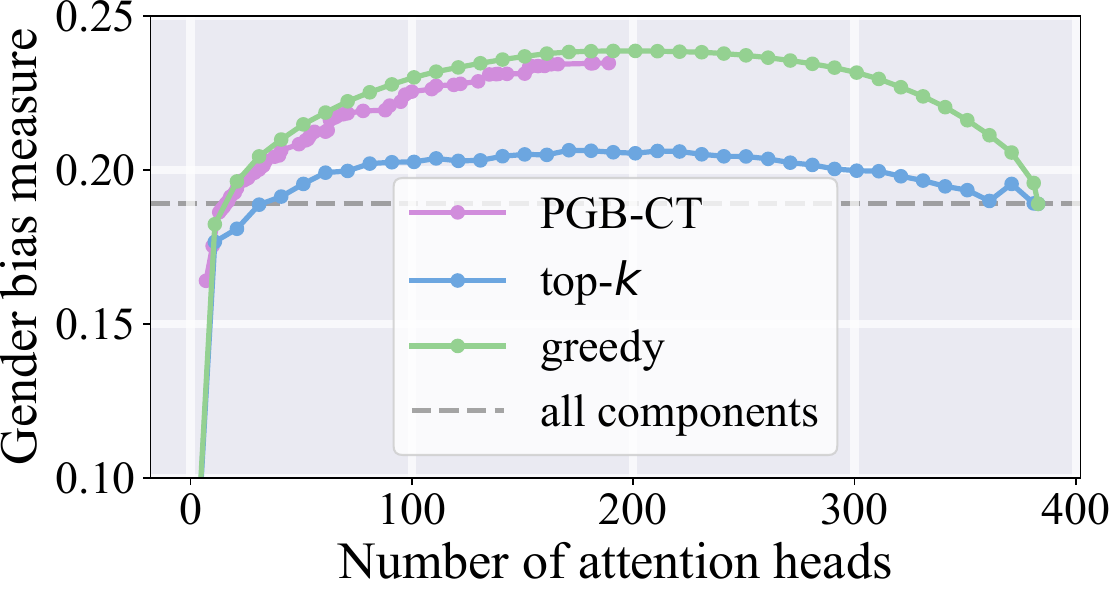}\hfill
\includegraphics[width=0.5\textwidth]{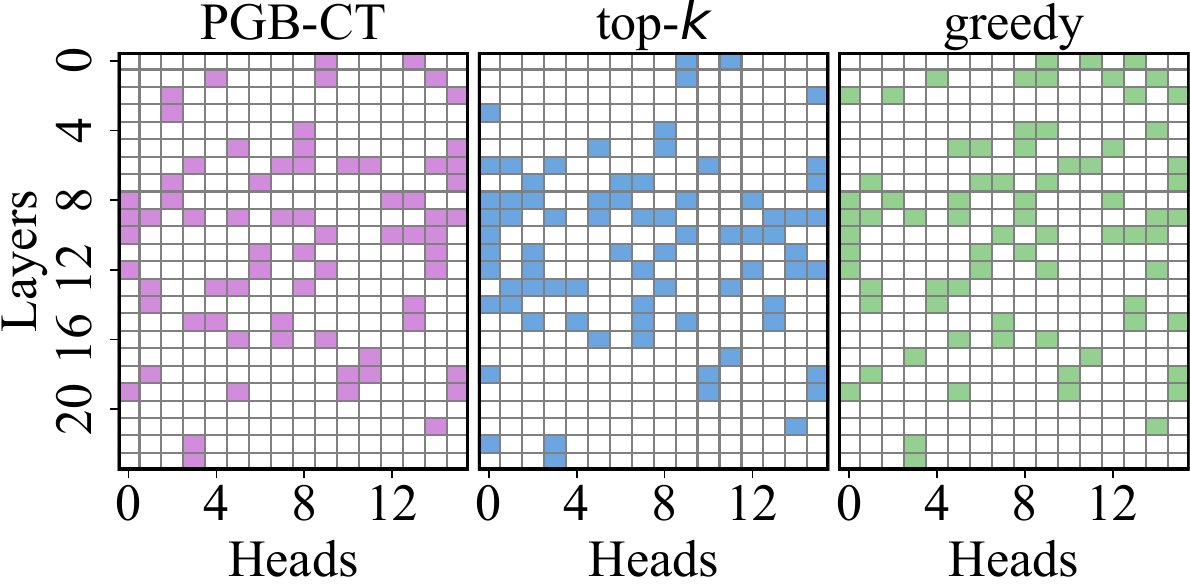}
    \caption{Results of selecting attention heads from GPT2-medium on the WinoGender dataset. \textit{Left:} Gender bias vs. number of attention heads. \textit{Right:} Selected attention heads with $20\%$ sparsity.}
    \label{fig:attention_addition_3}
\end{figure*}

\begin{figure*}[!tbp]
    \centering    \includegraphics[width=0.44\textwidth]{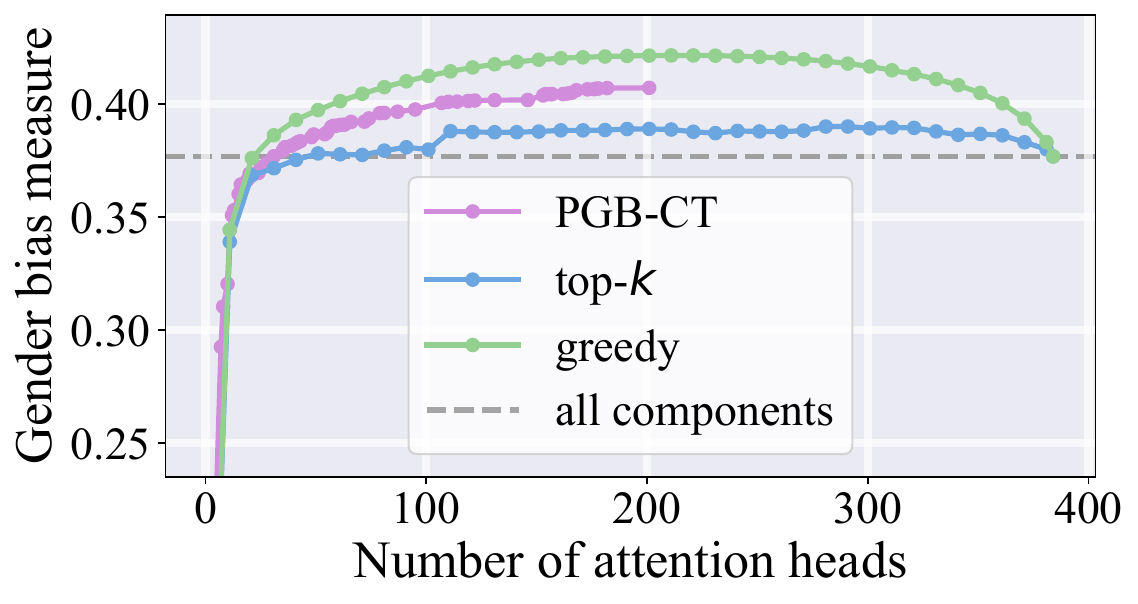}\hfill
\includegraphics[width=0.5\textwidth]{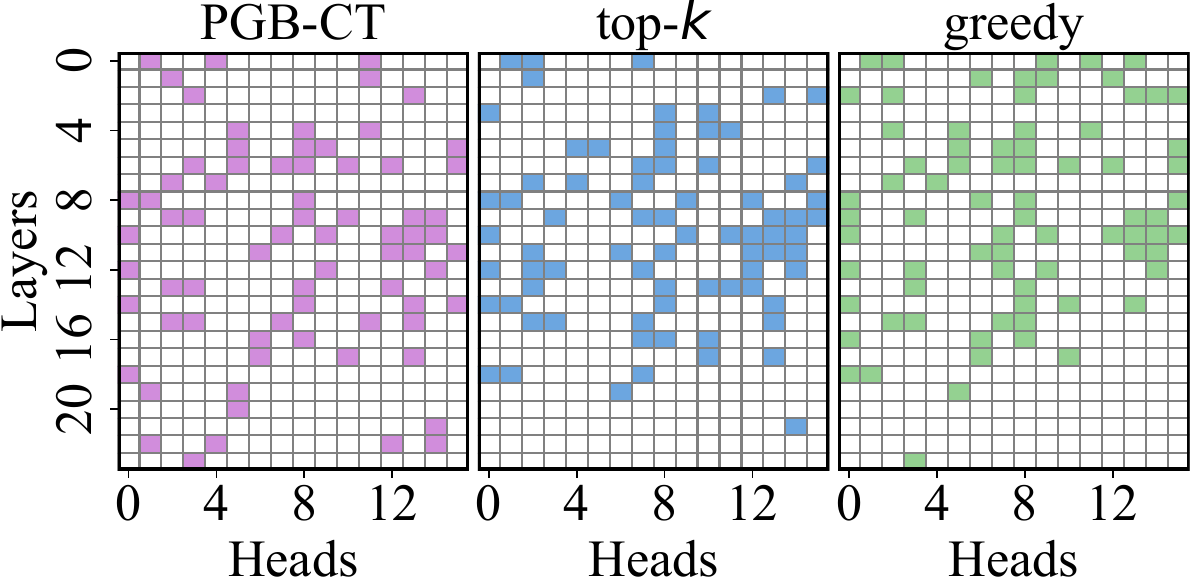}
    \caption{Results of selecting attention heads from GPT2-medium on the WinoBias dataset. \textit{Left:} Gender bias vs. number of attention heads. \textit{Right:} Selected attention heads with $20\%$ sparsity.}
    \label{fig:attention_addition_1}
\end{figure*}

\begin{figure*}[tbp]
  \centering
\centering    \includegraphics[width=0.42\textwidth]{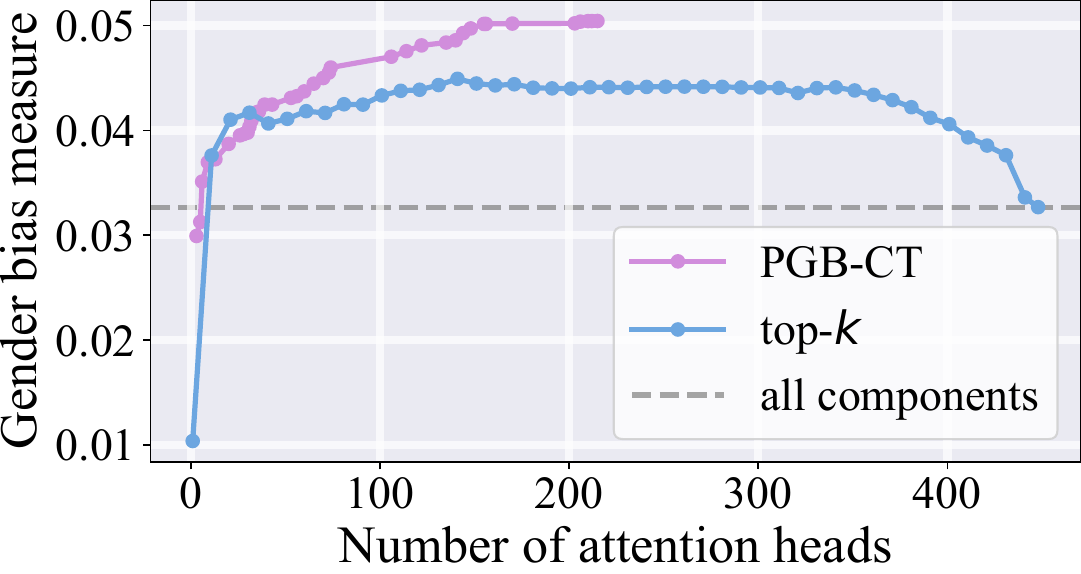}\hfill
    \includegraphics[width=0.42\textwidth]{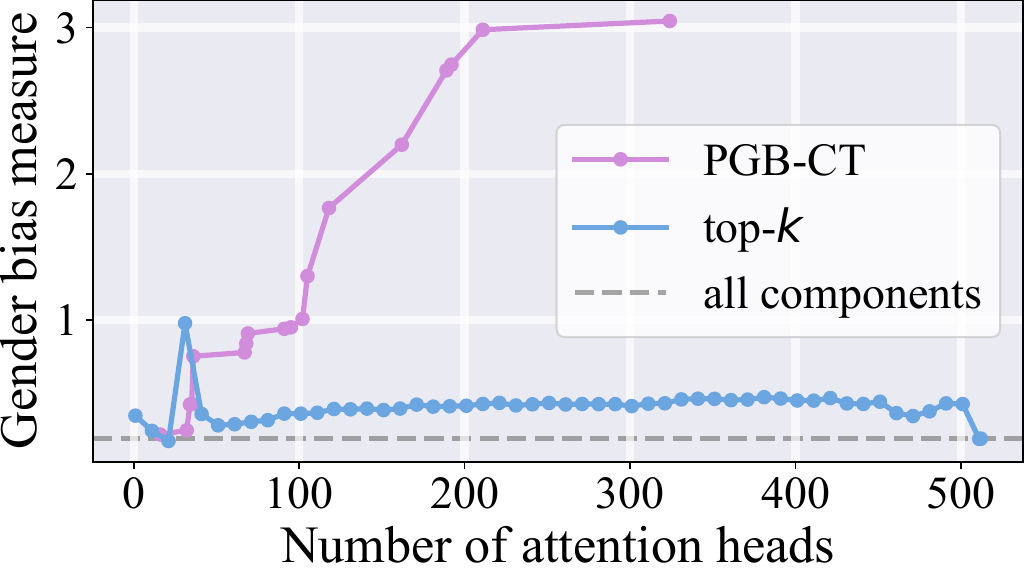} 
  \addtolength\abovecaptionskip{-0.05in}
  % \addtolength\belowcaptionskip{-0.2in}
  \caption{Results of selecting attention heads on the WinoGender dataset \textit{Left:} with Qwen3-1.7B. \textit{Right:} with Llama-3.2-1B.}
  \label{fig:new model}
\end{figure*}

\begin{figure*}[!htbp]
  \centering
\centering    \includegraphics[width=0.42\textwidth]{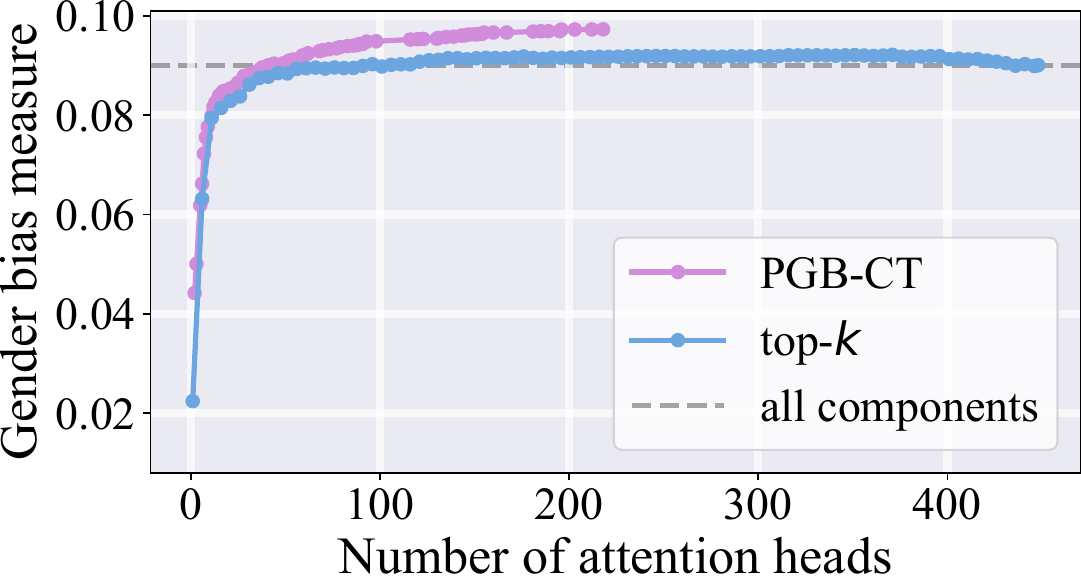}\hfill
    \includegraphics[width=0.42\textwidth]{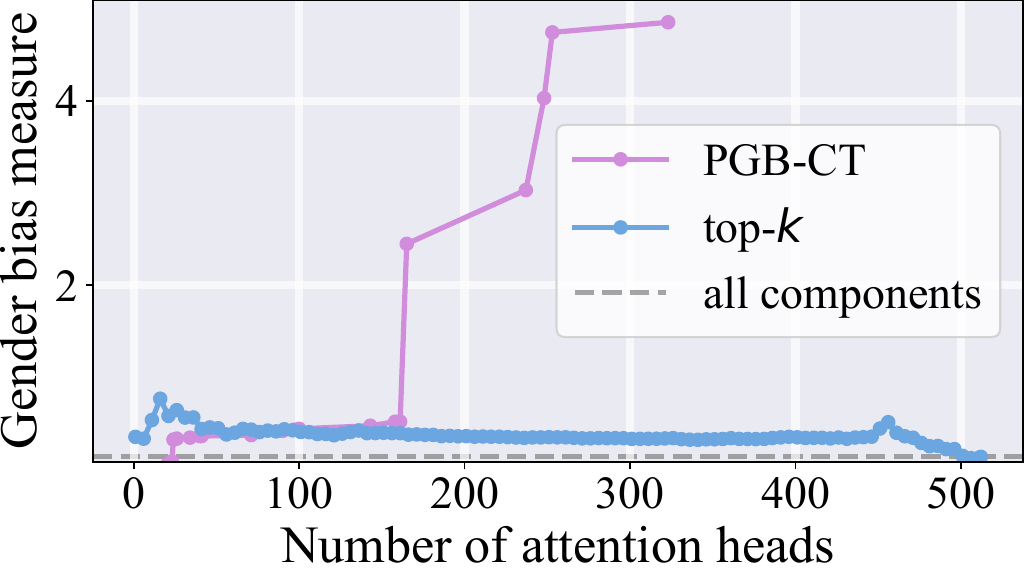} 
  \addtolength\abovecaptionskip{-0.05in}
  \caption{Results of selecting attention heads on the Winobias dataset: Gender bias measure vs.~number of neurons. \textit{Left:} with Qwen3-1.7B. \textit{Right:} with Llama-3.2-1B.}
  \label{fig:new model winobias}
\end{figure*}

\renewcommand{\arraystretch}{0.8}
\begin{table*}[th]
    \centering
    \begin{adjustbox}{max width=\textwidth}
    \begin{tabular}{ll|cccc|c}
        \toprule
        \multicolumn{2}{l|}{Model / Algorithm} 
          & \multicolumn{4}{c|}{Sparsity} 
          & Time \\
        \cmidrule(lr){3-6} \cmidrule(lr){7-7}
         &  & $\leq$2.5\% & $\leq$5\% & $\leq$10\% & $\leq$15\% & hours \\
        \midrule
        % --- DistilGPT2 ---
        \multirow{3}{*}{DistilGPT2} 
          & random   & 0.36 & 0.72 & 1.59 & 3.02 & 0.21 \\
          & top-$k$  & \textbf{14.86} & 47.23 & 103.79 & 149.28 & 0.97 \\
          & \algoname & 11.76 & \textbf{58.33} & \textbf{313.17} & \textbf{2308.08} & \textbf{0.02} \\
        \midrule
        % --- GPT2-small ---
        \multirow{3}{*}{GPT2-small} 
          & random   & 0.34 & 0.56 & 1.54 & 2.87 & 0.32 \\
          & top-$k$  & 30.23 & 84.71 & 127.32 & 140.80 & 2.97 \\
          & \algoname & \textbf{47.54} & \textbf{146.40} & \textbf{1320.25} & \textbf{8654.82} & \textbf{0.03} \\
        \midrule
        % --- GPT2-medium ---
        \multirow{3}{*}{GPT2-medium} 
          & random   & 0.25 & 0.48 & 1.35 & 2.77 & 0.71 \\
          & top-$k$  & 61.63 & 103.18 & 125.45 & 125.03 & 17.37 \\
          & \algoname & \textbf{1311.80} & \textbf{6.8e7} & \textbf{6.8e7} & \textbf{6.8e7} & \textbf{0.07} \\
        \midrule
        % --- GPT2-large ---
        \multirow{2}{*}{GPT2-large} 
          & random   & 0.24 & 0.39 & 0.85 & 1.57 & 1.35 \\
          & \algoname & \textbf{437.33} & \textbf{1.7e6} & \textbf{6.9e9} & \textbf{6.9e9} & \textbf{0.14} \\
        \midrule
        % --- GPT2-xl ---
        \multirow{2}{*}{GPT2-xl} 
          & random   & 0.28 & 0.63 & 1.46 & 3.12 & 2.52 \\
          & \algoname & \textbf{18852} & \textbf{1.9e5} & \textbf{6.1e10} & \textbf{6.6e11} & \textbf{0.25} \\
          \midrule
          \multirow{2}{*}{Qwen3-1.7B}
          & random & 2.15 & 2.88 & 3.54 & 3.55 &2.04  \\
          & \algoname &  \textbf{6.1e10}& \textbf{2.1e11} & \textbf{2.1e11} &  \textbf{2.1e11} & \textbf{0.20} \\
          \midrule
          \multirow{2}{*}{Llama3.2-1B} 
          & random & 1.13 & 1.51 & 1.94 & 2.03 &0.80  \\
          & \algoname & \textbf{3470.68}& \textbf{1.1e4} &\textbf{ 2.2e4} & \textbf{2.5e4} & \textbf{0.09}  \\
        \bottomrule
    \end{tabular}
    \end{adjustbox}
    \caption{
    Averaged metric $\ell(\mcD,\bm)$ and execution time on the Professions dataset for selecting MLP neurons across different sparsity levels, algorithms, and LLMs. The best results are highlighted in \textbf{bold}.}
    \label{tab:results_MLP}
\end{table*}
\renewcommand{\arraystretch}{1.0}

\subsection{Hyperparameter}
To allow exploration on both sides, we initialize the masking vector $\bm$ with a value of 0.5. Since the regularization parameters control the gradient's push toward sparsity and binarization, we found that setting them too high at the beginning hurts performance. Consequently, we adopt a linear scheduling strategy for the initial regularization parameters $\lambda_1$ and $\lambda_2$, where the initial value increases linearly with the number of epochs. We performed a grid search for both $\lambda_1$ and $\lambda_2$ in the range $[0,0.1]$, as our transformed metric lies in the range $(0,1)$ and we do not need large values. We set the number of epochs to 30 for the Professions dataset and 15 for the other three datasets. 
For the attention head experiments, we tuned the learning rate to 0.1, and for the MLP neuron experiments, we set it to 0.5 to ensure faster convergence. Finally, we employed the Adam optimizer with its default parameters. We also do truncation after each epoch for \algoname{} to evaluate its performance.

In the VBD dataset used in \emph{the Attention Heads with MLP Neurons} paragraph and \ref{sec:addVBN}, we use a learning rate of 0.01 and the Adam optimizer with default parameters. The number of epochs is set to be $1000$.

\subsection{Details for toy examples}
\label{sec:toy}
Figure~\ref{fig:results_toy} illustrates the deviation from linearity when intervening on two components simultaneously. Using the notations defined in Section~\ref{sec:llm} and Section~\ref{sec:CT}, we first compute the average metric $\ell(\mcD,\bm)$ for all combinations of intervention target $(i,j)$ for $i,j \in[N]$, where $N=36$ is the number of attention heads in three layers (left figure), and $N=12$ is the number of MLP layers in GPT2-small. Then the value at $(i,j)$-th position in the figure is calculated by subtracting the effects of individual components from their combined effect.
\begin{equation}
\ell(\mcD, \bm_{\{i,j\}}) - \ell(\mcD, \bm_{\{i\}}) - \ell(\mcD, \bm_{\{j\}}) \ ,
\end{equation} 
where the vector $\bm_{\mcS}\in\{0,1\}^N$ defined over a set $\mcS$  takes value 1 at the $k$-th component when $k \in \mcS$, and 0 otherwise.

\section{Additional Experimental Results}
\label{sec:additionalexperiments}
Additional experiments are organized as follows. Appendix~\ref{sec:ablation study} presents an ablation study on the penalty terms~$\lambda_1$ and~$\lambda_2$. Appendices~\ref{sec:attention} and~\ref{sec:additionalneuron} report supplementary results for attention heads and MLP neurons, respectively. Appendix~\ref{sec:finding} includes more discussion. Furthermore, we include the ablation study in the Appendix~\ref{sec:ablation study} and comparison to DCM in the Appendix~\ref{sec:DCM}. Section~\ref{sec:sparcity} discusses the binarization of $\mathbf{m}$ and the associated sparsity analysis.
Finally, Section~\ref{sec single-component} discusses the comparison of our algorithm to baselines when $S=1$.

\renewcommand{\arraystretch}{0.5}
\begin{table*}[th]
    \centering
    \begin{adjustbox}{max width=\textwidth}
    \begin{tabular}{ll|cccc|c}
        \toprule
        \multicolumn{2}{l|}{Model / Algorithm} 
          & \multicolumn{4}{c|}{Sparsity} 
          & Time \\
        \cmidrule(lr){3-6} \cmidrule(lr){7-7}
         &  & $\leq$2.5\% & $\leq$5\% & $\leq$10\% & $\leq$15\% & hours \\
        \midrule
        % --- DistilGPT2 ---
        \multirow{3}{*}{DistilGPT2} 
          & random   & 2.35 & 18.07 & 111.89 & 37.67 & 1.98 \\
          & top-$k$  & 5.72 & 24.77 & 24.99 & 22.45 & 10.08 \\
          & \algoname & \textbf{15.63} & \textbf{62.22} & \textbf{286.34} & \textbf{1141.76} & \textbf{0.07} \\
        \midrule
        % --- GPT2-small ---
        \multirow{3}{*}{GPT2-small} 
          & random   & 3.88 & 13.75 & 239.06 & 122.11 & 3.99 \\
          & top-$k$  & 27.49 & 66.02 & 107.60 & 197.02 & 32.35 \\
          & \algoname & \textbf{755.95} & \textbf{1790.33} & \textbf{8376.24} & \textbf{21155.90} & \textbf{0.13} \\
        \bottomrule
    \end{tabular}
    \end{adjustbox}
    \vspace{-0.05 in}
    \caption{Averaged metric $\ell(\mcD,\bm)$ and execution time on the Factual (CounterFact) dataset for selecting MLP neurons across different sparsity levels, algorithms, and LLMs. The best results are highlighted in \textbf{bold}.}
    \label{tab:results_factual}
\end{table*}
\renewcommand{\arraystretch}{1.0}

\begin{figure*}[tbp]
    \centering
    \includegraphics[width=0.42\textwidth]{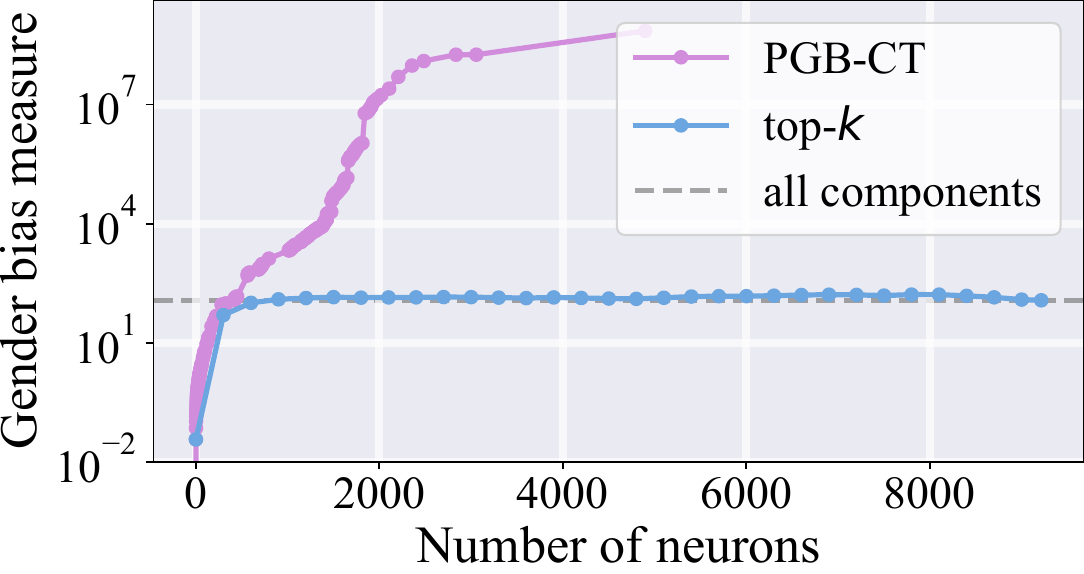}\hfill
    \includegraphics[width=0.5\textwidth]{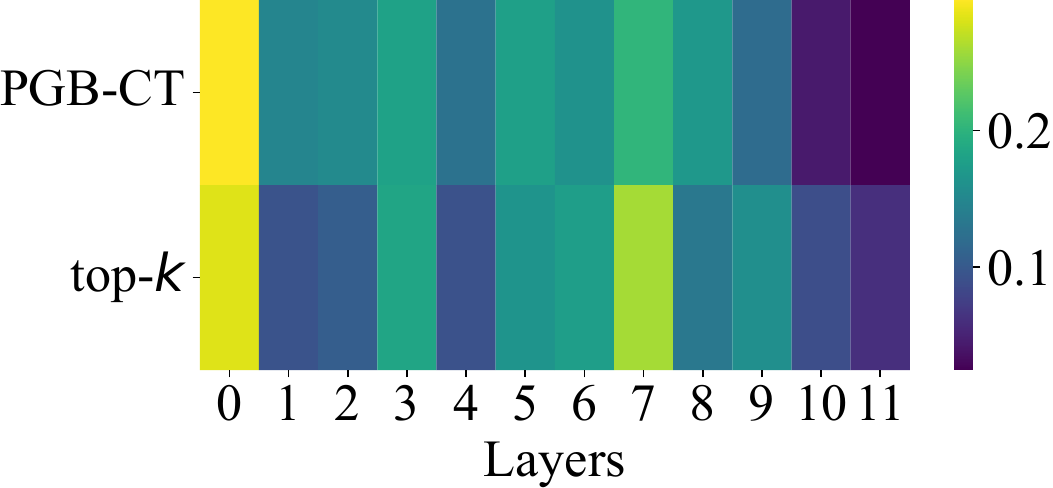}
    \caption{Results of selecting MLP neurons on the Professions dataset with GPT2-small. \textit{Left:} Gender bias measure vs.~number of neurons. \textit{Right:} Fraction of Selected neurons in each layer with 15\% sparsity.}
    \label{fig:mlp_additional_1}
\end{figure*}

\begin{figure*}[!tbp]
    \centering
    \includegraphics[width=0.42\textwidth]{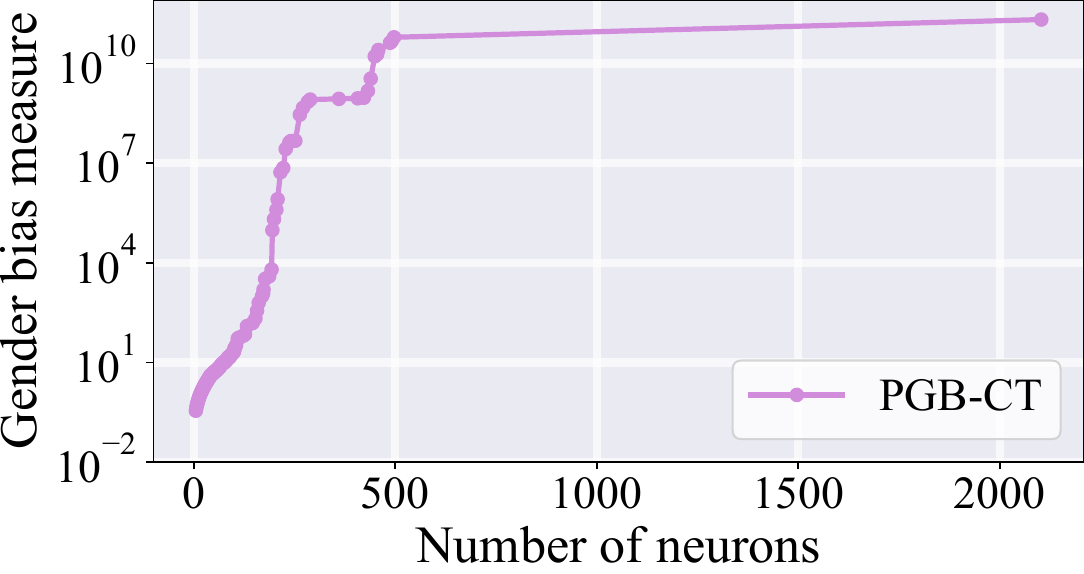}\hfill
    \includegraphics[width=0.42\textwidth]{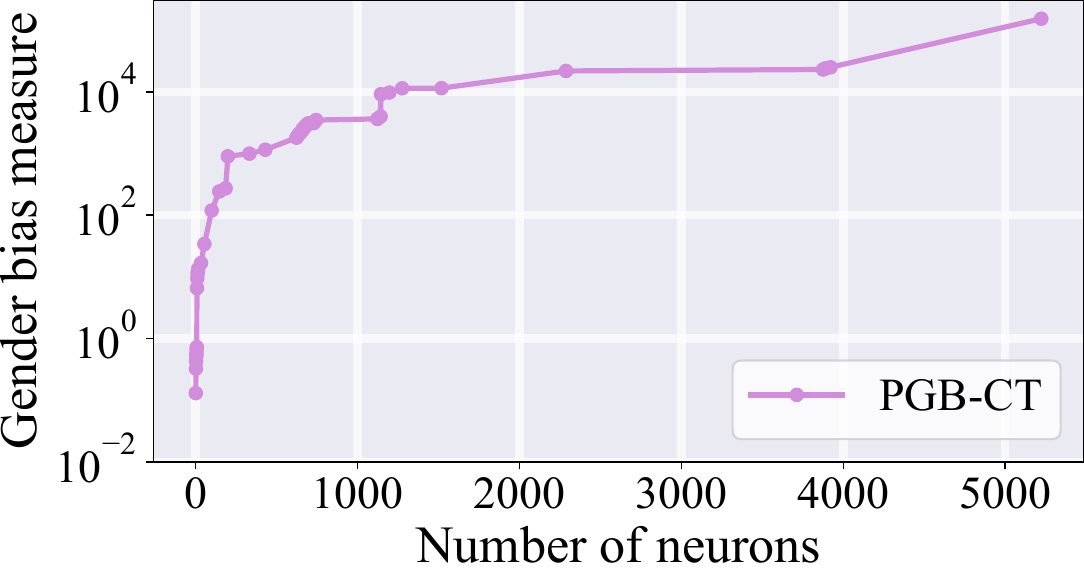}
    \caption{Results of selecting MLP neurons on the Professions dataset: Gender bias measure vs.~number of neurons. \textit{Left:} with Qwen3-1.7B. \textit{Right:} with Llama-3.2-1B.}
    \label{fig:mlp_additional_new}
\end{figure*}

\subsection{Attention heads}
\label{sec:attention}
Figures~\ref{fig:attention_addition_1}-\ref{fig:new model winobias} show the additional figures for the experiments on the attention heads. Specifically,  Figure~\ref{fig:attention_addition_2} shows the results with GPT2-small on the WinoGender dataset; Figure~\ref{fig:attention_addition_3} shows the results with GPT2-medium on the WinoGender dataset; Figure~\ref{fig:attention_addition_1} shows the results with GPT2-medium on the WinoBias dataset. Figure~\ref{fig:new model} extends to Qwen3‑1.7B and Llama‑3.2‑1B models, where we omit the greedy algorithm because it is too time‑consuming. Figure~\ref{fig:new model winobias} shows the results with Qwen and Llama on the WinoBias dataset. We find that, in all cases, \algoname{} outperforms the top-$k$, and is close to greedy when applicable.

\subsection{Neurons}
\label{sec:additionalneuron}
The results for the Professions dataset are summarized in Table~\ref{tab:results_MLP} and the results for the CounterFact dataset are summarized in Table~\ref{tab:results_factual}. In addition, Figure~\ref{fig:mlp_additional_1} shows the results with GPT2-small in the Professions dataset; Figure~\ref{fig:mlp_additional_2} shows the results with distilGPT2 in the CounterFact dataset. Figure~\ref{fig:mlp_additional_new} shows the results with Qwen and Llama models. It is evident that \algoname{} shows an exponential increase in metric with the number of components, whereas the other algorithms cannot.

\subsection{Other findings}
We list some additional findings below.
\label{sec:finding}
\begin{itemize}[leftmargin=0.15in, itemsep=0.03in, topsep=0.04in]
    \item \emph{Model structure and size.} Consistent with the single‑component causal‑tracing results in~\cite{vig2020investigating}, we find that larger models generally display stronger gender bias. More intriguingly, bias patterns vary across architectures: e.g., the bias on Llama-3.2-1B is larger than Qwen3-1.7B. As illustrated in Figure~\ref{fig:new model} (right), the top‑$k$ baseline fails on this architecture,  highlighting the need for multi-component causal tracing.
    \item \emph{Highly Non-Linear Effect of MLP Neurons.} We observe a surprising non-linearity in the effect of MLP neurons (unlike attention heads). Specifically, the target metric increases at varying exponential rates as the sparsity of the selected set changes.
    \item \emph{``Half-half'' rule:} Across all evaluated models, we observe that about 50\% of the attention heads or neurons collectively boost the target metric and the other 50\% negatively impact it.
\end{itemize}

\begin{figure}[!htbp]
    \centering
    \vspace{-0.05 in}
    \includegraphics[width=0.9\columnwidth]{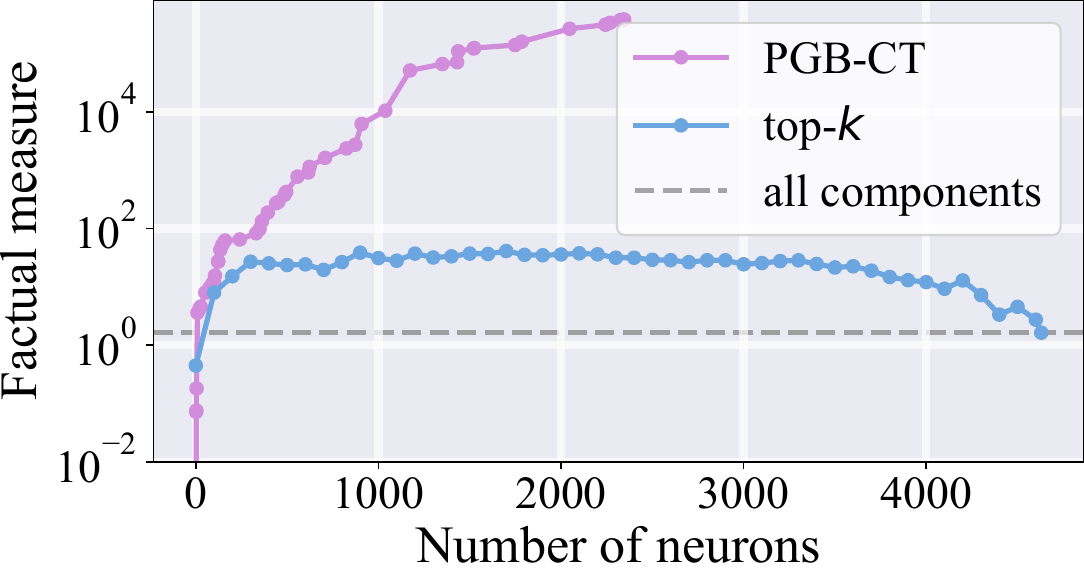}
    \caption{Results of Factual locating measure vs.~number of MLP neurons on the CounterFact dataset with GPT2-small.}
    \label{fig:mlp_additional_2}
\end{figure}

\begin{figure*}[!htbp]
    \centering    \includegraphics[width=0.42\textwidth]{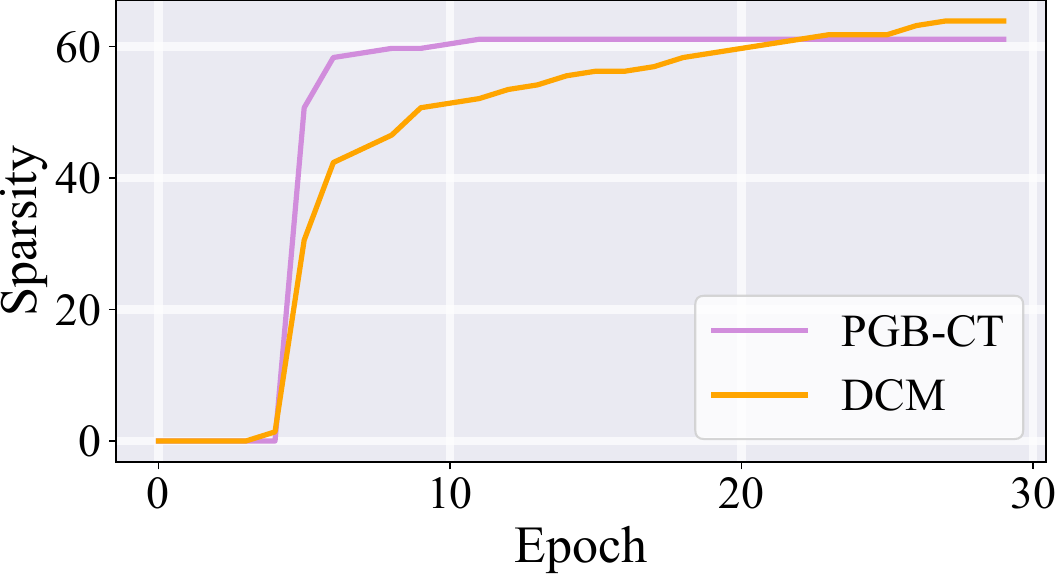}\hfill
\includegraphics[width=0.42\textwidth]{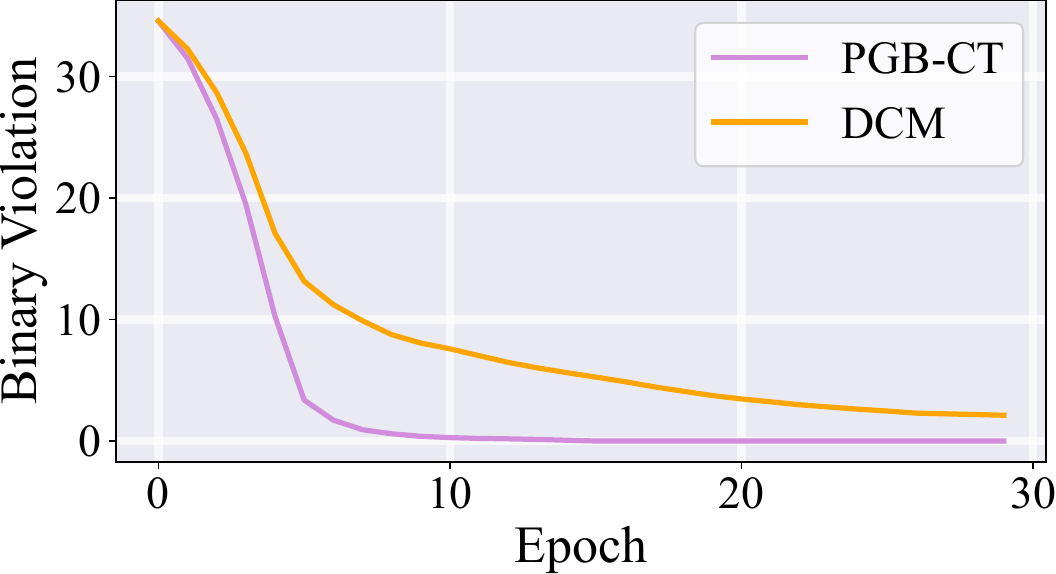}
    \caption{Comparison of properties of $\mathbf{m}$ of attentions on WinoGender dataset.}
    \label{fig:DCM}
\end{figure*}

\begin{figure*}[!tbp]
    \centering
\includegraphics[width=0.42\textwidth]{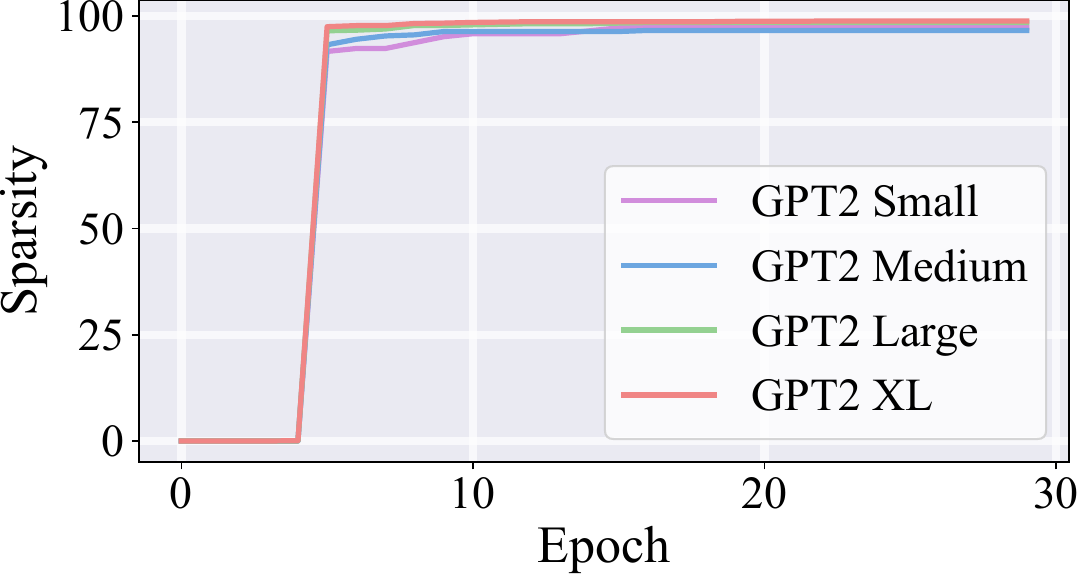}\hfill
\includegraphics[width=0.42\textwidth]{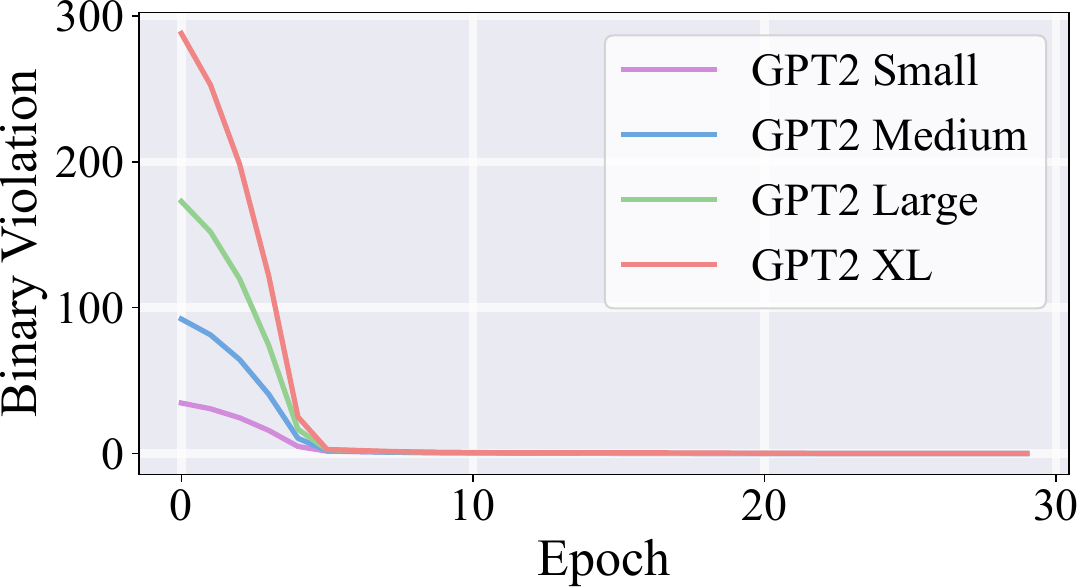}
    \caption{Properties of $\mathbf{m}$ of attentions on WinoGender dataset. \textit{Left:} Sparsity $S/N$. \textit{Right:} Violation of binary $\bm(\mathbf{1}-\bm)$.}
    \label{fig:m_additional}
\end{figure*}

\begin{figure}[!htbp]
    \centering
    \includegraphics[width=0.84\columnwidth]{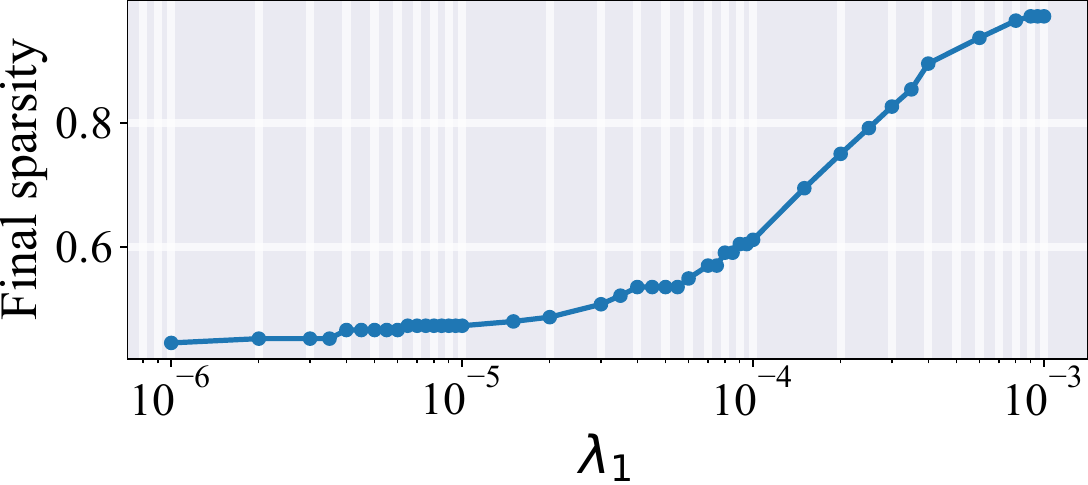}
    \caption{Results of sparsity vs.~$\lambda_1$ on the WinoGender dataset with GPT2-small.}
    \label{fig:m_addtional_2}
\end{figure}

\subsection{Ablation study}
\label{sec:ablation study}
Figure~\ref{fig:Ablation study} presents the ablation results in the WinoGender dataset using GPT2-small, with all other hyperparameters and grid search settings remaining the same. The figure shows that removing the first $\ell_1$ penalty term prevents the algorithm from achieving a sparse solution, while removing the second penalty term degrades its performance.

Besides, we also include the DCM algorithm~\cite{davies2023discovering}, which can find some intervention combinations that are close to those found by our algorithm \algoname{} but converges only to certain sparsity levels, making the algorithm harder to tune and less practical. We discuss this further in Appendix~\ref{sec:DCM}.

\subsection{Why DCM Fails}
\label{sec:DCM}

As shown in Appendix~\ref{sec:ablation study}, Desiderata-based Component Masking (DCM) can identify some subsets of components whose performance is close to that of \algoname{} at a certain sparsity level. However, in practice, DCM often fails to find subsets that satisfy the desired sparsity or metric level, resulting in suboptimal solutions. We attribute this behavior to its penalty function, which does not effectively binarize the mask weights. Figure~\ref{fig:DCM} compares the sparsity ratio $S/N$ and the binary violation term $\bm(\mathbf{1}-\bm)$ for \algoname{} (with $\lambda_1=0.0001,\lambda_2=0.001$) and DCM (with $\lambda=10^{-5}$). The left panel of Figure~\ref{fig:DCM} shows that both methods achieve similar sparsity after truncation. However, the right panel shows that the binary violation for DCM does not converge to zero, and when it attempts to approach zero, the sparsity increases further.

This is because, although DCM rounds the final mask weights to binary (similar to \eqref{eq:discrete}), the amount of “binary violation” $\bm(\mathbf{1}-\bm)$ before truncation is still crucial. A larger average value of $\bm(\mathbf{1}-\bm)$ indicates that many mask entries are closer to 0.5 at the end of training. When such a mask is truncated to a binary one, many components undergo large changes, which can cause a significant drop in the metric compared to the soft mask. Therefore, the higher binary violation observed for DCM in the right panel suggests a larger potential loss in the metric after the final truncation step.

\subsection{Distribution of weights \texorpdfstring{$\mathbf{m}$}{m} in training}
\label{sec:sparcity}
The \algoname{} algorithm encourages the values in $\mathbf{m}$ to be exactly 0 or 1, resulting in binary selections. Experimentally, we found that binarization is achieved after the last epoch by choosing an appropriate~$\lambda_2$, and $\lambda_1$ controls sparsity more strongly. Figure~\ref{fig:m_additional} shows results from trials on different types of GPT‑2 models with $\lambda_1=\lambda_2=1\mathrm{e}{-3}$. We observe that all models behave similarly: larger models exhibit slightly larger binary violations because they contain more parameters, yet they ultimately converge to higher sparsity since the penalty parameter does not scale with model size. Furthermore, Figure~\ref{fig:m_addtional_2} shows how sparsity changes when varying $\lambda_1$ on the WinoGender dataset with GPT2-small. We can vary from low sparsity to high sparsity with tuning $\lambda_1$.

\renewcommand{\arraystretch}{0.75}
\begin{table*}[th]
    \centering
    \begin{adjustbox}{max width=\textwidth}
    \begin{tabular}{ll|cccc|c}
        \toprule
        \multicolumn{2}{l|}{Model / Algorithm}
          & \multicolumn{4}{c}{Dataset}
           \\
        \cmidrule(lr){3-5}
         &  & WinoGender & WinoBias & Professions  \\
        \midrule
        % --- DistilGPT2 ---
        \multirow{2}{*}{DistilGPT2} 
          & Optimal  & 0.0059 & 0.013 & 0.049  \\
          & \algoname & 0.0059 & 0.012 & 0.049 &  \\
        \midrule
        % --- GPT2-small ---
        \multirow{2}{*}{GPT2-small} 
          & Optimal & 0.019 & 0.105 & 0.037  \\
          & \algoname & 0.019 & -- & 0.037  \\
        \midrule
        \multirow{2}{*}{GPT2-medium} 
          & Optimal & 0.189 & 0.377 & 0.023  \\
          & \algoname & -- & -- & 0.023  \\
        \bottomrule
    \end{tabular}
    \end{adjustbox}
    \caption{Performance comparison in the single-component setting ($S=1$): \algoname{} and the optimum.}
    \label{tab:results_s1}
\end{table*}
\renewcommand{\arraystretch}{1.0}

\subsection{Performance on single-component \texorpdfstring{$S=1$}{S=1}}
\label{sec single-component}
When $S=1$, our method reduces to traditional single-component causal tracing (likewise for top-$k$ and greedy baselines). Empirically, its performance is nearly identical to single-component methods. In the main paper, for attention-head selection on GPT-2-small with WinoBias (leftmost point in Figure~\ref{fig:results_attention}) and MLP-neuron selection on GPT-2-medium with Profession (leftmost point in Figure~\ref{fig:results_MLP}), our algorithm selects the same optimal component as both top-$k$ and greedy. For larger models and harder settings (e.g., CounterFact), optimization can struggle to converge exactly to $S=1$, leading to slight suboptimality. We note that the DCM algorithm cannot meet these sparsity requirements. See Table~\ref{tab:results_s1} for detailed results.

\subsection{Additional results on VBD dataset}
\label{sec:addVBN}
Figure~\ref{fig:VBD_7B} further shows the experiment results on jointly analyzing attention heads and whole MLP blocks on LLaMA-7B.

When the sparsity level is $S=15$, we find the following combination for LLaMA-7B: Attention Heads 8.1, 12.11, 13.0, 13.15, 13.19, 14.19, 15.9, 15.23, 17.25, 18.12, 19.1, 19.6, 21.12, 25.2, and MLP block 10.

\label{sec:vadadd}
\begin{figure}[!t]
    \centering
\includegraphics[width=0.8\columnwidth]{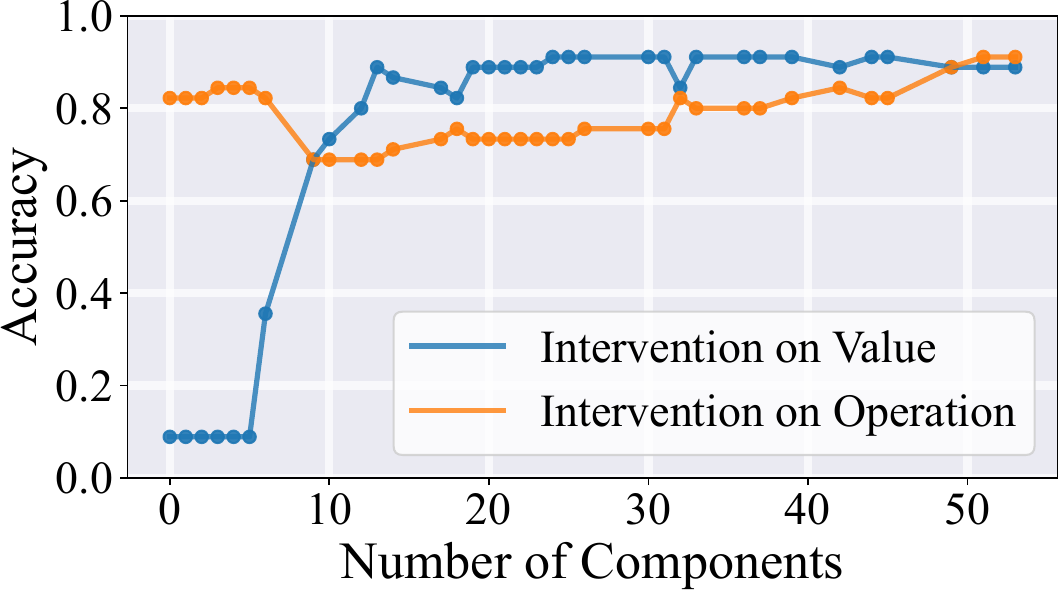}
    \caption{Results on the VBD dataset with LLaMA-7B, under interventions on the values and operations in the equations.}
    \label{fig:VBD_7B}
\end{figure}

\section{Proof of theorem~\ref{thm:complexity}}
\label{pf:complexity}
\begin{proof}
We notice $C_N(S)$ takes the form
\begin{equation}
\label{eq:SNC}
    C_N(S) \;\triangleq\; \sum_{c=1}^{S} \binom{N}{c} \ .
\end{equation}
We prove each regime separately.
\paragraph{1. $S$ is a fixed constant.}
When $S$ does \emph{not} depend on $N$, the sum $C_N(S)$ in \eqref{eq:SNC} has exactly $S$ terms. Each term $\binom{N}{c}$ is a polynomial in $N$ of degree $c$. Indeed, for fixed $c$,
\begin{align}
\binom{N}{c} 
&\;=\;\frac{N(N-1)\cdots(N-c+1)}{c!} \\
&\;=\; \mcO\bigl(N^c\bigr) \ .
\end{align}
Therefore,
\begin{align}
C_N(S) &\;=\;\sum_{c=1}^{S} \binom{N}{c} \\ 
& \;=\; \binom{N}{1} + \binom{N}{2} + \cdots + \binom{N}{S}\\
& \;=\; \mcO(N) + \mcO(N^2) + \cdots +\mcO(N^S) \nonumber\\
& \;=\; \mcO(N^S)  \ .
\end{align}

Moreover, since the last term is included in the sum, we have
\begin{equation}
C_N(S)
\;\geq\; \binom{N}{S} \;=\;
\Omega(N^S) \ ,
\end{equation}
Hence, we obtain
\begin{equation}
    C_N(S) \;=\; \Theta(N^S) \ .
\end{equation}

\paragraph{2. $S = \alpha N$ for some $\alpha \in (0,1)$.}
We first state the detailed form for the bound of $C_N$ as follows
\begin{equation}
C_N(S)
=
\begin{cases}
\Theta\!\left(\dfrac{2^{N H_2(\alpha)}}{\sqrt{N}}\right) \;,
& 0<\alpha<\frac12\\[1.2em]
\Theta(2^N) \;,
& \frac12 \leq \alpha < 1
\end{cases} \ ,
\end{equation}
where $H_2(\alpha)$ is the binary entropy function
\begin{equation}
H_2(\alpha)
= -\alpha \log_2(\alpha) - (1-\alpha)\log_2\bigl(1-\alpha\bigr) \ .
\end{equation}

\medskip
Consider the sum defined in
\begin{equation}
C_N(S)
\;=\;\sum_{c=1}^{\lfloor \alpha N \rfloor} \binom{N}{c} \ .
\end{equation}
The binomial coefficients $\binom{N}{c}$ are unimodal in $c$, achieving their maximum when $c$ is around $N/2$. If $\alpha \le 1/2$, the largest term in the sum up to $c = \lfloor \alpha N\rfloor$ is near $c = \lfloor \alpha N\rfloor$. If $\alpha > 1/2$, by symmetry, the main contribution is instead near the central region and we will prove that it scales to the order of $2^N$.

\medskip

\noindent \emph{(a) Dominant term and Stirling's approximation.}
Using Stirling's approximation
\begin{equation}
n! \;\sim\; \sqrt{2\pi\,n}\,\Bigl(\frac{n}{e}\Bigr)^{n}
\quad \text{as }n \to \infty \ ,
\end{equation}
for $k = \lfloor \alpha N \rfloor$, we have
\begin{align}
&\binom{N}{k}
=\frac{N!}{k!\,(N-k)!}\\
&\approx \!
\frac{\sqrt{2\pi N}\,\bigl(\tfrac{N}{e}\bigr)^{N}}
{\sqrt{2\pi k}\,\bigl(\tfrac{k}{e}\bigr)^{k}\;\sqrt{2\pi(N-k)}\,\bigl(\tfrac{N-k}{e}\bigr)^{N-k}} \ .
\end{align}

Canceling common factors, we obtain the approximation 

\begin{align}
\binom{N}{k}
\;&\approx\;
\frac{N^N}{k^k\,(N-k)^{N-k}} \nonumber\\
&\quad \times \frac{1}{\sqrt{2\pi\,k\,(N-k)/N}}\\
&=\;\frac{1}{\sqrt{2\pi\,\alpha(1-\alpha)\,N}} \nonumber\\
&\quad \times \frac{1}{\bigl(\alpha^\alpha\,(1-\alpha)^{1-\alpha}\bigr)^N} \ .
\end{align}
Noting that
$\alpha^\alpha\,(1-\alpha)^{1-\alpha} = 2^{-H_2(\alpha)}$. 
Hence, we get
\begin{equation}
\binom{N}{\alpha N}
\;\approx\;
\frac{2^{NH_2(\alpha)}}{\sqrt{2\pi\,\alpha(1-\alpha)\,N}} \ .
\end{equation}
For $\alpha = 1/2$, this is the classical
\begin{equation}
\binom{N}{N/2}
\;\sim\;\frac{2^{N}}{\sqrt{\pi N/2}} \ .
\end{equation}

\medskip

\noindent \emph{(b) Sum of binomial coefficients.}
We now distinguish the cases $0<\alpha<1/2$ and $\alpha\ge 1/2$.

If $0<\alpha<1/2$, let $k=\lfloor\alpha N\rfloor$. Then the largest term
in the sum is $\binom{N}{k}$. Moreover, for $1\le j\le k$,
\begin{equation}
\frac{\binom{N}{j-1}}{\binom{N}{j}}
\;=\;
\frac{j}{N-j+1}
\;\leq\;
\frac{k}{N-k+1} \ .
\end{equation}
Since $k/N\sim\alpha<1/2$, there exists a constant $r<1$ such that
\begin{equation}
    \frac{k}{N-k+1} \;\leq\; r \ ,
\end{equation}
for all sufficiently large $N$. Therefore the preceding terms are bounded
by a geometric series for fixed $k$
\begin{align}
\binom{N}{k}
&\;\leq\;
C_N(k)\\
&\;\leq\;
\binom{N}{k}\sum_{t=0}^{k}r^t
\;\leq\;
\frac{1}{1-r}\binom{N}{k} \ .
\end{align}
Hence $C_N(k)=\Theta\bigl(\binom{N}{k}\bigr)$. Using the Stirling estimate
from part (a), we get
\begin{equation}
C_N(\lfloor\alpha N\rfloor)
\;=\;
\Theta\!\left(
\frac{2^{N H_2(\alpha)}}{\sqrt{N}}
\right) \ .
\end{equation}

If $\alpha=1/2$, then the sum contains one half of the binomial mass, up to
the central coefficient, and hence
\begin{equation}
C_N(\lfloor N/2\rfloor)
\;=\;
\Theta(2^N) \ .
\end{equation}

If $\alpha>1/2$, then the sum contains the previous half-mass case and is
still bounded above by the full binomial sum:
\begin{equation}
    \sum_{k=0}^{N}\binom{N}{k}\;=\;2^N \ .
\end{equation}
Hence
\begin{equation}
C_N(\lfloor\alpha N\rfloor)
\;=\;
\Theta(2^N),
\quad \frac12<\alpha<1 \ .
\end{equation}
Thus the tight scaling is
\begin{equation}
C_N(S)
=
\begin{cases}
\Theta\!\left(
\dfrac{2^{N H_2(\alpha)}}{\sqrt{N}}
\right),
& 0<\alpha<\frac12 \\[1em]
\Theta(2^N),
& \frac12\le \alpha<1
\end{cases} \ .
\end{equation}

This completes the proof of both claims in the theorem.
\end{proof}

\end{document}